\documentclass[conference]{IEEEtran}
\IEEEoverridecommandlockouts

\usepackage{amsmath,amssymb,amsthm}
\usepackage{mathtools}
\usepackage{algorithm}
\usepackage{algpseudocode}
\usepackage{graphicx}
\usepackage{booktabs}
\usepackage{array}
\usepackage{multirow}
\usepackage{xcolor}
\usepackage{tikz}
\usetikzlibrary{arrows.meta, positioning, fit, calc, shapes.geometric}
\usepackage{url}
\usepackage{hyperref}
\hypersetup{colorlinks=true, linkcolor=black,
            citecolor=black, urlcolor=black}

\newtheorem{theorem}{Theorem}
\newtheorem{proposition}{Proposition}
\newtheorem{lemma}{Lemma}

\theoremstyle{definition}
\newtheorem{definition}{Definition}
\newtheorem{assumption}{Assumption}
\theoremstyle{remark}
\newtheorem{remark}{Remark}

\newcommand{\R}{\mathbb{R}}
\newcommand{\E}{\mathbb{E}}
\newcommand{\Prob}{\mathbb{P}}
\newcommand{\Sphere}{\mathbb{S}^{d-1}}
\newcommand{\WAND}{\textsc{Wand}}
\newcommand{\med}{\mathrm{med}}
\newcommand{\MAD}{\mathrm{MAD}}
\DeclareMathOperator*{\argmax}{arg\,max}

\title{Witnesses Explain Anomalies}

\author{
\IEEEauthorblockN{Lamine Diop}
\IEEEauthorblockA{\textit{EPITA Research Laboratory}, Le Kremlin-Bic\^etre FR-94276, France, lamine.diop@epita.fr}
}

\begin{document}
\maketitle

\begin{abstract}
Unsupervised anomaly detection scores each point of an unlabelled,
contaminated sample in a single pass, and increasingly must also
explain \emph{why} a point is flagged. Yet the dominant detectors give
a score with no account of which features drive it, and explanations
are bolted on post-hoc with SHAP or LIME, which re-query the detector
thousands of times per point and only \emph{approximate} it. We
introduce \WAND{}, an unsupervised tabular anomaly detector that is
\emph{explainable by design}. \WAND{} organises its computation around
\emph{directions} on the unit sphere, scoring each point by how far
its projection escapes a sub-Gaussian extreme-value baseline. The
originality of our approach is that the \emph{witness directions} that
flag a point, being vectors in feature space, \emph{are} its
explanation, a per-feature attribution obtained at no cost over scoring
and, since the score is differentiable, recoverable by gradients.
Scoring is \emph{linear} in the sample size, and a probe-efficiency
bound guarantees every anomaly a witness, hence an explanation. Across
47 ADBench datasets \WAND{} attains the best mean Friedman rank at
ROC-AUC parity with 16 unsupervised baselines, so the gain is
interpretability at no accuracy cost; its native explanations are more
accurate and faithful than post-hoc SHAP/LIME and ECOD at a fraction
of the query cost. \WAND{} is thus a practical, interpretable solution
for explainable anomaly detection.
\end{abstract}

\begin{IEEEkeywords}
anomaly detection, explainable AI, witness directions, feature attribution, differentiable algorithms
\end{IEEEkeywords}

\section{Introduction}
\label{sec:intro}

\begin{figure*}[t]
\centering
\includegraphics[width=0.82\textwidth]{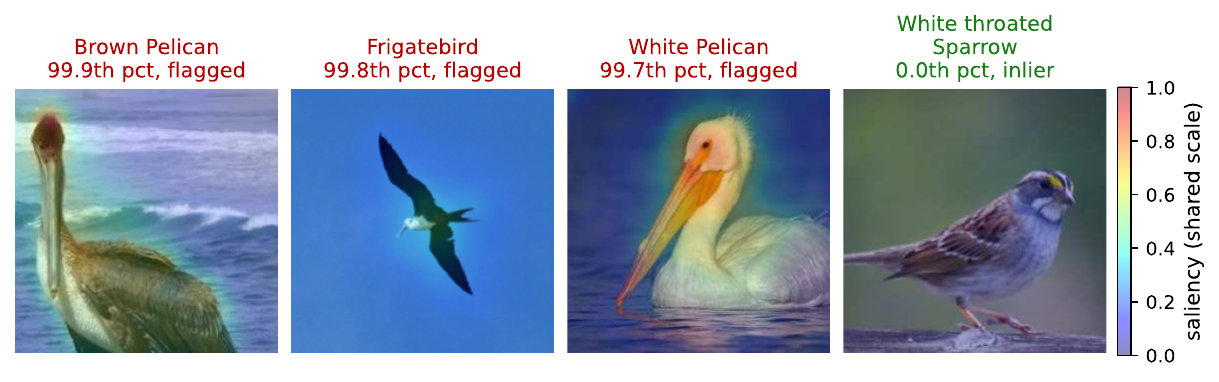}
\caption{\textbf{Pixel-level explanation} (\emph{where} in the image).
Saliency $|\partial\,\textsc{Wand}\text{ score}/\partial\,\text{pixels}|$
for three flagged \textsc{AnoCUB} anomalies and a normal inlier (far
right), on a shared scale, via autograd through a frozen ResNet-18
encoder (SmoothGrad): strong and localised on the anomalies, weak on the
inlier. \textsc{Wand} detects here at AUC $0.99$ (no whitening).
Contrast the encoder-free \emph{concept-level} view
(Fig.~\ref{fig:anocub}), which names \emph{which} concepts are
responsible.}
\label{fig:pixsal}
\end{figure*}

We study unsupervised anomaly detection on tabular data: assign each
point of an unlabelled, contaminated sample a real-valued score in a
single pass, and, increasingly, explain \emph{why} a point is
flagged. The dominant detectors (Isolation Forest~\cite{liu2008iforest},
LOF~\cite{breunig2000lof}, OCSVM~\cite{scholkopf2001ocsvm},
KNN~\cite{ramaswamy2000knn}, PCA~\cite{shyu2003pca},
ECOD~\cite{li2022ecod}/COPOD~\cite{li2020copod}) return a score but no
account of \emph{which} features drive it; explanations are bolted on
post-hoc with SHAP~\cite{lundberg2017shap}/LIME~\cite{ribeiro2016lime},
which re-query the detector thousands of times per point and only
\emph{approximate} it. We introduce \WAND{}, a detector that is
explainable by construction, down to raw-pixel saliency when an
encoder is available (Fig.~\ref{fig:pixsal}).

\WAND{} organises its computation around \emph{directions} on the
unit sphere: for each, it measures how far the projected sample's
extreme exceeds a sub-Gaussian baseline (calibrated by median/MAD,
robust to up to half the data being contaminated), and aggregates these
excesses into a per-point score. Anomalies have low halfspace
depth~\cite{tukey1975mathematics}, so one or more directions expose
them as extreme, and those \emph{witness} directions are, literally,
the explanation: vectors in feature space whose large coordinates name
the responsible features. A per-feature attribution thus falls out of
scoring at no extra cost, and, since the score is differentiable, the
same explanation is recoverable by gradients (and the score can serve
as a loss in a learnable pipeline). The geometry also bounds cost on two axes. Scoring runs in $O(Knd)$
time, \emph{linear} in the sample size; and the \emph{number of
directions} $K$ needed to expose every anomaly scales with the answer
size (anomaly count and depth margin), not the sample
size~\cite{chazelle1993optimal}. This latter is a \emph{probe-efficiency}
guarantee: each direction still requires a full scan of the data, so it
bounds a budget rather than promising sub-linear runtime, and is most
useful in low-to-moderate dimension (Section~\ref{sec:conclusion}).

The main contributions of the paper are as follows:
\begin{itemize}
  \item We propose \WAND{}, an unsupervised tabular anomaly detector
    that is \emph{explainable by design}. It scores each point by
    aggregating, over directions on the unit sphere, the excess of the
    projected sample's extreme over a sub-Gaussian baseline, and the
    \emph{witness directions} that flag a point, read in feature space,
    \emph{are} its per-feature explanation, obtained at no cost over
    scoring (Section~\ref{sec:witness}). Since the score is
    differentiable, the same explanation is also recoverable by
    gradients.
  \item We establish the guarantees of the method. We prove that
    scoring is \emph{linear} in the sample size, and that the number of
    directions needed to expose every anomaly is bounded by the anomaly
    count and a halfspace-depth margin, independently of the sample
    size (Theorem~\ref{thm:output-sensitive}); equivalently, this
    bounds the budget at which every anomaly is guaranteed a witness,
    hence an explanation. We pair this with a median/MAD-calibrated
    tail statistic of $1/(d{+}1)$ breakdown
    (Theorem~\ref{thm:robust}).
  \item We present an extensive empirical study on 47 ADBench datasets,
    comparing detection against sixteen unsupervised baselines and
    explanation quality against post-hoc (SHAP, LIME) and AD-native
    (ECOD) explainers, under both synthetic ground truth and real
    deletion/insertion faithfulness. \WAND{} is competitive for
    detection while producing more accurate and more faithful
    explanations at a fraction of the query cost, and stays effective
    on heavy-tailed inliers and against deep detectors
    (Section~\ref{sec:exp}).
\end{itemize}

The remainder of the paper is organised as follows.
Section~\ref{sec:related} reviews related work on unsupervised and
explainable anomaly detection. Section~\ref{sec:bg} introduces the
basic definitions and the formal problem statement. We present
\WAND{} and its directional-witness explanations in
Section~\ref{sec:method}, and analyse the method's guarantees in
Section~\ref{sec:theory}. We evaluate our approach in
Section~\ref{sec:exp} and conclude in Section~\ref{sec:conclusion}.

\section{Related Work}
\label{sec:related}

We group prior work by what it consumes and what guarantees it
provides, then identify the gap \WAND{} fills.

\paragraph*{Unsupervised shallow detectors.}
The classical workhorses learn a score directly from the contaminated
sample, in three flavours. \emph{Isolation/proximity}-based scores
(Isolation Forest~\cite{liu2008iforest}, LOF~\cite{breunig2000lof},
OCSVM~\cite{scholkopf2001ocsvm}, KNN~\cite{ramaswamy2000knn},
PCA-reconstruction~\cite{shyu2003pca},
INNE~\cite{bandaragoda2018inne}) measure how isolated a point is
from its neighbourhood. \emph{Distribution-aware tail estimators}
(HBOS~\cite{goldstein2012hbos}, ECOD~\cite{li2022ecod},
COPOD~\cite{li2020copod}, kernel density~\cite{parzen1962kde},
LODA~\cite{pevny2016loda}) read off the inlier distribution through
marginals or copulas. \emph{Geometric and ensemble} variants
(ABOD~\cite{kriegel2008abod}, COF~\cite{tang2002cof},
SOD~\cite{kriegel2009sod}, LSCP~\cite{zhao2019lscp},
SUOD~\cite{zhao2021suod}) sharpen the score with directional or
subspace structure. All these methods consume the sample directly
and score it in one pass, matching our regime, but each pays
linear cost in $n$ regardless of how few anomalies are actually
present; none gives a probe-count guarantee tied to the anomaly
count.

\paragraph*{Robust statistics and projection-pursuit ancestors.}
Two threads underpin our per-direction statistic. Tukey halfspace
depth~\cite{tukey1975mathematics} formalises being an extreme along
some direction with $1/(d+1)$
breakdown~\cite{donoho1992breakdown}, but exact computation
$\Omega(n^{d-1})$~\cite{aloupis2006geometric} has kept it out of
practical pipelines. Projection
pursuit~\cite{friedman1974projection} optimises
an index over directions for non-Gaussianity; we re-use the
directional formulation with a MAD-calibrated tail-excess statistic
of $1/2$ per-direction
breakdown~\cite{rousseeuw1993alternatives}. The closest
projection-pursuit predecessors, LODA~\cite{pevny2016loda} and
PIDForest~\cite{gopalan2019pidforest}, neither calibrate
against an explicit extreme-value baseline nor give an
output-sensitive probe bound; we list more recent tree- and OT-based
detectors among the limitations of our baseline~coverage~below.

\paragraph*{Deep detectors.}
A separate family trains a network, typically on an assumed-clean
sample, deep one-class~\cite{ruff2018deepsvdd},
reconstruction~\cite{zong2018dagmm},
adversarial/transformer~\cite{tuli2022tranad}, self-supervised
contrastive~\cite{yin2024mcm}, graph~\cite{goodge2022lunar}, and
diffusion/flow~\cite{livernoche2024dte} variants; \textsc{TabADM}~\cite{zamberg2023tabadm}
is a diffusion density model with a rejection scheme that tolerates a
contaminated training set. These need a separate scoring pass and give no
native explanation; we list them in
Table~\ref{tab:competitors} for positioning and compare against deep
\emph{unsupervised} baselines (AutoEncoder, VAE, Deep SVDD) in
Section~\ref{sec:exp}.

\paragraph*{Explaining anomaly detectors.}
Once a detector flags a point, practitioners ask which features are
responsible. The default answer is \emph{post-hoc, model-agnostic}
attribution: SHAP~\cite{lundberg2017shap} (Shapley values estimated by
sampling feature coalitions) and LIME~\cite{ribeiro2016lime} (a local
linear surrogate fitted to perturbations) both treat the detector as a
black box and re-query it hundreds to thousands of times per explained
point, returning an \emph{approximation} of its behaviour whose quality
degrades with dimension and budget. Anomaly-specific explainers largely
inherit this stance, attributing a score to features
post-hoc~\cite{takeishi2019shapley,antwarg2021explaining} or isolating
outlying subspaces~\cite{kriegel2009sod,micenkova2013subspace}. In
contrast, a handful of detectors are explanatory \emph{by
construction}: marginal methods like ECOD/COPOD expose a per-feature
tail probability, and subspace/tree methods point at the axes that
isolate a point. \textsc{SYRAN}~\cite{hossain2026syran} goes furthest,
learning closed-form \emph{symbolic invariants} whose violation is the
explanation, a \emph{global} account, whereas \WAND{} reads a
\emph{per-point} attribution free from the score. \WAND{} is of this second kind, but its witnesses
are arbitrary directions, not single features, so it explains anomalies
that no axis-aligned method can name; and the explanation is read
directly from the score, not a sampled approximation.

\paragraph*{Positioning.}
Differentiable surrogates for sort/argmax~\cite{diffsort2020} supply our
soft-extreme operator, and output-sensitive analysis~\cite{chazelle1993optimal}
motivates the probe-count (not runtime) budget. \WAND{} stays in the
unsupervised shallow regime (contaminated, single-pass, clean-data-free) but
uniquely combines \emph{all} of: native sampling-free witness attribution, a
probe-efficient budget with a coverage guarantee, a MAD-calibrated
$1/(d{+}1)$-breakdown statistic, and end-to-end differentiability
(Table~\ref{tab:competitors}).

\begin{table}[t]
  \centering
  \scriptsize
  \setlength{\tabcolsep}{2pt}
  \renewcommand{\arraystretch}{1.0}
  \caption{Method profiles. \emph{Native expl.}: built-in, exact
  feature attribution ($\bullet$ = axis-only); \emph{Probe eff.}:
  anomaly-count-bounded probe budget. The shallow family is our
  comparison set.}
  \label{tab:competitors}
  \begin{tabular}{@{}p{0.56\linewidth} c c c c c@{}}
    \toprule
     & Clean & Train & Diff. & \shortstack{Native\\expl.} & \shortstack{Probe\\eff.}\\
    \midrule
    \multicolumn{6}{l}{\textit{Unsupervised shallow (comparison set)}} \\
    IForest, LOF, OCSVM, KNN, PCA      & no & no & no & no & no \\
    HBOS, ECOD, COPOD, KDE             & no & no & no & ($\bullet$) & no \\
    ABOD, COF, SOD, INNE, LODA, LSCP   & no & no & no & ($\bullet$) & no \\
    PIDForest~\cite{gopalan2019pidforest}
                                       & no & no & no & ($\bullet$) & no \\
    \midrule
    \multicolumn{6}{l}{\textit{Deep (different regime)}} \\
    Deep one-class \cite{ruff2018deepsvdd}
       & yes & \checkmark & \checkmark & no & no \\
    Reconstruction \cite{zong2018dagmm}
       & yes & \checkmark & \checkmark & no & no \\
    Adversarial / transformer \cite{tuli2022tranad}
       & yes & \checkmark & \checkmark & no & no \\
    Self-supervised \cite{yin2024mcm}
       & yes & \checkmark & \checkmark & no & no \\
    Graph / diffusion \cite{goodge2022lunar,livernoche2024dte}
       & yes & \checkmark & \checkmark & no & no \\
    \midrule
    \textbf{\WAND{} (ours)}        & \textbf{no} & no & \checkmark & \textbf{\checkmark} & \checkmark \\
    \bottomrule
  \end{tabular}
\end{table}

\section{Background and Setting}
\label{sec:bg}

\subsection{Problem Statement}

We observe $X = \{x_{1},\dots,x_{n}\} \subset \R^{d}$ drawn i.i.d.\
from a \emph{contaminated} mixture
\begin{equation}
P = (1-\varepsilon)\,P_{\mathrm{in}} + \varepsilon\,P_{\mathrm{out}},
\qquad \varepsilon \in (0, 1/2),
\label{eq:huber}
\end{equation}
where $P_{\mathrm{in}}$ is the (unknown) inlier law and
$P_{\mathrm{out}}$ is arbitrary. The unsupervised anomaly-detection
task is to produce a score $s : \R^{d}\to\R_{\ge 0}$ such that, with
high probability, the $k = \lceil \varepsilon n\rceil$ highest-scored
points coincide with the outlier subset
$A = \{i : x_{i} \sim P_{\mathrm{out}}\}$. We evaluate $s$ by ROC-AUC,
which is invariant to monotone re-scaling.

Throughout, $u\in\Sphere := \{u \in \R^{d} : \|u\|_{2} = 1\}$ denotes a
unit direction, $z_{i}(u) := u^{\top}x_{i}\in\R$ the corresponding
projection, and $\med, \MAD$ the (univariate) median and median
absolute deviation. We write $(\cdot)_{+}$ for the positive part and
$\E,\Prob$ for probability and expectation under $P$.

\subsection{Inlier Assumption}

We make a single distributional assumption on $P_{\mathrm{in}}$, weaker
than Gaussianity and stable under affine transformations:

\begin{assumption}[Isotropic-tail inlier]\label{ass:subgauss}
$P_{\mathrm{in}}$ is centered and $\sigma^{2}$--sub-Gaussian, i.e.\
for every direction $u\in\Sphere$ and every $\lambda\in\R$,
\[
  \E_{P_{\mathrm{in}}}\!\left[\exp(\lambda\, u^{\top}X)\right]
  \le \exp\!\left(\tfrac{\lambda^{2}\sigma^{2}}{2}\right).
\]
\end{assumption}

Assumption~\ref{ass:subgauss} is the standard high-dimensional
prerequisite for projection-based methods
\cite{vershynin2018hd}; it is satisfied by Gaussian mixtures,
log-concave laws, sub-exponential tails after a robust standardisation,
and any \emph{bounded} distribution. It is not satisfied by power-law
tails (e.g.\ Pareto), but in such regimes the median/MAD rescaling
inside \WAND{} effectively truncates extreme inlier draws back
into a sub-Gaussian regime.

\subsection{The Sub-Gaussian Anti-Concentration Baseline}

The key quantity that drives both the algorithm and the analysis is the
following population--level extreme-value bound:

\begin{proposition}[Extreme-value baseline]\label{prop:cdn}
Under Assumption~\ref{ass:subgauss} and zero contamination,
\[
  \frac{\max_{i\le n}\,u^{\top}x_{i}\;-\;\med(z(u))}{\MAD(z(u))}
  \;\le\; c_{d}(n) + O_{P}\!\left(\tfrac{1}{\sqrt{\log n}}\right),
\]
uniformly in $u\in\Sphere$, where
\begin{equation}
  c_{d}(n) \;:=\; \sqrt{2\log n} \;+\; \tfrac{\log 2}{\sqrt{2\log n}}
\label{eq:cdn}
\end{equation}
is the sub-Gaussian extreme \emph{envelope} (used as a deterministic
upper bound, not as a tight Gumbel limit). In practice $c_{d}(n)$ is
replaced by the empirical-null quantile $q_{0}$ of
\eqref{eq:score}, so the precise constant inside $c_{d}(n)$ does not
affect implemented results.
\end{proposition}

\begin{proof}
See the supplementary material.
\end{proof}

Uniformity in $u$ holds because the bound uses only the
direction-independent proxy $\sigma$; numerically $c_{d}(n)$ runs
$3.7\!\to\!5.3$ for $n{=}10^{3}\!\to\!10^{6}$ (a clean max sits a few MADs
above the median). We use \eqref{eq:cdn} as a
constant in the per-direction score \eqref{eq:tau-mad}, and replace it
by an empirical bootstrap quantile $q_{0}$ when calibrating direction
weights (Section~\ref{sec:method}).

\subsection{Anomaly Definition}

Proposition~\ref{prop:cdn} motivates a margin-based notion of anomaly
that is direction-existential rather than direction-universal:

\begin{definition}[$\tau$-margin anomaly]\label{def:anomaly}
For a margin $\tau > 0$, a point $x\in X$ is a $\tau$-margin anomaly
iff there exists a direction $u\in\Sphere$ such that
\[
  \frac{|u^{\top}x - \med(z(u))|}{\MAD(z(u))} \;\ge\; c_{d}(n) + \tau.
\]
We write $A_{\tau}\subseteq X$ for this set, and abbreviate $k = |A_{\tau}|$.
\end{definition}

Definition~\ref{def:anomaly} is the finite-sample analogue of low
\emph{halfspace depth} \cite{tukey1975mathematics,donoho1992breakdown}:
the witness direction $u$ defines a closed halfspace
$H_{u}(x) = \{y : u^{\top}y \ge u^{\top}x\}$ that separates $x$ from the
inlier bulk by a $\tau$-margin in MAD-units. The advantage over plain
Tukey depth is that we only require \emph{one} witness direction per
anomaly, which is what makes the output-sensitive sample-complexity
bound of Theorem~\ref{thm:output-sensitive} possible.

\noindent\emph{On the margin $\tau$.} It is the only unobserved quantity
in the framework, yet we never specify it: the threshold $c_{d}(n)+\tau$
is replaced by an empirical null quantile $q_{0}$ from a Gaussian copy of
$X$ (Algorithm~\ref{alg:anticop}); $\tau$ controls the
inlier-tail/outlier gap and appears only in the theoretical bound.

\section{Method}
\label{sec:method}

\subsection{Pipeline Overview}

\WAND{} produces an anomaly score $s : X \to \R_{\ge 0}$ by four
stages, each addressing a specific failure mode of na\"ive halfspace
probing. Figure~\ref{fig:pipeline} summarises the data flow.

\medskip\noindent
\textbf{(P1) Per-direction tail-excess (Section~\ref{sec:per-dir}).}
For each candidate direction $u$, every point $x_{i}$ receives a
non-negative excess $\tau_{i}(u)$ that measures how far
$u^{\top}x_{i}$ lies beyond the sub-Gaussian baseline of
Proposition~\ref{prop:cdn}. A complementary 1D $k$-spacings component
(Section~\ref{sec:spacing}) handles multi-modal projections that defeat
plain MAD-$z$.

\medskip\noindent
\textbf{(P2) Pathway split (Section~\ref{sec:pathway}).}
A uniform pool of probes on $\Sphere$ and a fixed set of axis-aligned
probes are scored \emph{separately}, then combined by a guarded
additive rule, the key engineering step that prevents a single
noisy axis from dominating the score.

\medskip\noindent
\textbf{(P3) Seed averaging (Section~\ref{sec:ensemble}).}
To suppress Monte-Carlo variance, we average $S$ independently-seeded
passes.

\medskip
The complete pipeline is Algorithm~\ref{alg:anticop}. The whole score
is differentiable in the data $X$ through a soft-extreme surrogate
(Section~\ref{sec:diff}), enabling end-to-end use inside larger
trainable~systems.

\begin{figure}[t]
\centering
\resizebox{0.95\linewidth}{!}{%
\begin{tikzpicture}[
  >=Latex,
  font=\footnotesize,
  every node/.style = {font=\footnotesize},
  data/.style    = {draw, rounded corners=2pt, fill=blue!8,
                    minimum width=2.0cm, minimum height=0.8cm,
                    align=center, inner sep=3pt},
  stage/.style   = {draw, rounded corners=2pt, fill=orange!18,
                    minimum width=3.0cm, minimum height=0.9cm,
                    align=center, inner sep=3pt},
  branch/.style  = {draw, rounded corners=2pt, fill=green!12,
                    minimum width=3.0cm, minimum height=0.9cm,
                    align=center, inner sep=3pt},
  result/.style  = {draw, rounded corners=2pt, fill=red!12,
                    minimum width=2.0cm, minimum height=0.8cm,
                    align=center, inner sep=3pt},
  ar/.style      = {-{Latex[length=2mm]}, semithick},
  arlbl/.style   = {font=\scriptsize\itshape, fill=white, inner sep=1pt},
  tag/.style     = {font=\scriptsize\bfseries, color=black!70},
]

\node[stage] (P1) at (7.0, 0)
  {\textbf{(P1)} per-direction tail-excess\\
   $\tau_{i}(u) = \max\!\bigl(\tau^{\mathrm{mad}}_{i}(u),\, s_{u}\tau^{\mathrm{spc}}_{i}(u)\bigr)$};

\node[data] (X) at (7.0, -1.7) {data\\$X \in \R^{n \times d}$};
\draw[ar] (X.north) -- (P1.south) node[arlbl, midway, right] {project};

\node[fill=blue!8,    draw, minimum width=0.35cm, minimum height=0.25cm, inner sep=0] (sw1) at (1.4, -5.6) {};
\node[anchor=west, font=\scriptsize, inner sep=0pt] (lb1) at (sw1.east) {~input data};
\node[fill=orange!18, draw, minimum width=0.35cm, minimum height=0.25cm, inner sep=0] (sw2) at (1.4, -6.0) {};
\node[anchor=west, font=\scriptsize, inner sep=0pt] (lb2) at (sw2.east) {~processing stage};
\node[fill=green!12,  draw, minimum width=0.35cm, minimum height=0.25cm, inner sep=0] (sw3) at (1.4, -6.4) {};
\node[anchor=west, font=\scriptsize, inner sep=0pt] (lb3) at (sw3.east) {~probe pool};
\node[fill=red!12,    draw, minimum width=0.35cm, minimum height=0.25cm, inner sep=0] (sw4) at (1.4, -6.8) {};
\node[anchor=west, font=\scriptsize, inner sep=0pt] (lb4) at (sw4.east) {~output score};
\node[draw, rounded corners=2pt, fit=(sw1)(sw2)(sw3)(sw4)(lb1)(lb2)(lb3)(lb4),
      inner xsep=4pt, inner ysep=3pt, label={[font=\scriptsize\itshape,
      yshift=-2pt]above:legend}] {};

\node[branch] (rand) at (4.0, -2.6)
  {uniform-sphere probes\\$\{u_{1},\dots,u_{K}\}$};
\node[branch] (axis) at (10.0, -2.6)
  {axis-aligned probes\\$\{e_{1},\dots,e_{d}\}$};
\draw[ar] (P1.south west) -- ++(0,-0.30) -| (rand.north);
\draw[ar] (P1.south east) -- ++(0,-0.30) -| (axis.north);

\node[stage] (Srand) at (4.0, -4.4)
  {$s^{\mathrm{rand}}_{i} = \sum_{k} w_{k}\tau_{i}(u_{k})/\sum_{k} w_{k}$};
\node[stage] (Saxis) at (10.0, -4.4)
  {$s^{\mathrm{axis}}_{i} = \sum_{j} w_{j}\tau_{i}(e_{j})/\sum_{j} w_{j}$};
\draw[ar] (rand) -- (Srand);
\draw[ar] (axis) -- (Saxis);

\node[stage] (mix) at (7.0, -6.2)
  {\textbf{(P2)} guarded additive mix\\
   $\widetilde{s}^{\mathrm{rand}} + \lambda\,\widetilde{s}^{\mathrm{axis}}$,
   $\lambda \in (0, 1]$};
\draw[ar] (Srand.south) -- ++(0,-0.45) -| (mix.north);
\draw[ar] (Saxis.south) -- ++(0,-0.45) -| (mix.north);

\node[stage] (ens) at (7.0, -7.9)
  {\textbf{(P3)} seed averaging\\$\bar{s}(x_{i}) = \tfrac{1}{S}\sum_{t=1}^{S} s^{(t)}(x_{i})$};
\draw[ar] (mix) -- (ens);

\node[result] (score) at (2.0, -7.9) {final\\score $\bar{s}(x_{i})$};
\draw[ar] (ens.west) -- (score.east);

\node[tag] at (4.0, -1.7) {primary path};
\node[tag] at (10.0, -1.7) {marginal path};

\end{tikzpicture}}
\caption{\WAND{} pipeline: per-direction tail-excess (P1), guarded
mix of uniform-sphere and axis pathways (P2), seed averaging (P3).}
\label{fig:pipeline}
\end{figure}
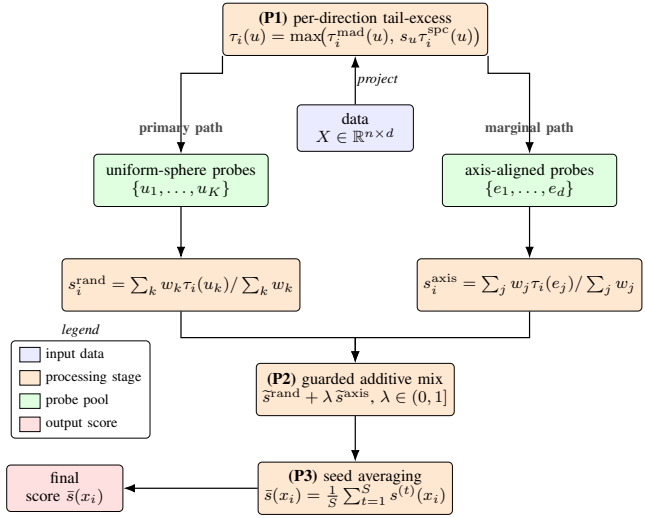

\subsection{Per-Direction Tail-Excess}
\label{sec:per-dir}

The atomic statistic of \WAND{} is, for each direction $u\in\Sphere$
and projection $z_{i}(u) = u^{\top}x_{i}$,
\begin{align}
  r_{i}(u) &= \frac{z_{i}(u) - \med(z(u))}{\MAD(z(u))},
  & \tau^{\mathrm{mad}}_{i}(u) &= \bigl(|r_{i}(u)| - c_{d}(n)\bigr)_{+}.
\label{eq:tau-mad}
\end{align}
We use $|r_{i}|$ (not $r_{i}$) so $u$ and $-u$ contribute symmetrically,
and the median/MAD rescaling gives breakdown $1/2$ per
direction~\cite{donoho1992breakdown}, no single outlier can
inflate the scale. The direction-level excess is
\begin{equation}
  \Delta^{\mathrm{mad}}(u;X) \;=\; \max_{i}\,\tau^{\mathrm{mad}}_{i}(u).
\label{eq:delta-mad}
\end{equation}
By Proposition~\ref{prop:cdn}, $\Delta^{\mathrm{mad}}(u;X) = o_{P}(1)$
uniformly over $u$ under the inlier null, so any direction with
$\Delta^{\mathrm{mad}}(u;X) \gg 1$ is statistical evidence of a
$\tau$-margin anomaly along $u$.

\noindent\emph{Why median/MAD.} Mean/std has breakdown $0$ (one moved
point can drive the standardised score to $0$ for \emph{all} others);
MAD's constant-factor efficiency loss under the Gaussian null
(ARE $\approx 0.37$) is absorbed by the empirical-null calibration of
$q_{0}$.

\subsection{Spacings: Multi-Modal Robustness}
\label{sec:spacing}

Equation~\eqref{eq:delta-mad} fails when $u^{\top}X$ is well-separated
into two or more modes, because the inter-mode gap inflates $\MAD$ and
crushes $|r_{i}|$ even for the rarest mode. To recover the rare-mode
signal we add a non-parametric \emph{$k$-spacings} component.

For each direction $u$, sort the projections $z_{(1)} \le \cdots \le
z_{(n)}$, and let $d_{k,i}(u)$ be the two-sided $k$-th order spacing,
\[
  d_{k,i}(u) \;=\; \min\!\Bigl(\,z_{(\pi_{i}+k)}-z_{(\pi_{i})},\;
                          z_{(\pi_{i})}-z_{(\pi_{i}-k)}\,\Bigr),
\]
with $\pi_{i}$ the rank of $z_{i}(u)$ and edge clamps at the
endpoints. A point sitting alone in a sparse region has a wide
$d_{k,i}$ relative to the median; this defines
\begin{equation}
  \tau^{\mathrm{spc}}_{i}(u)
  \;=\; \Bigl(\log\,\tfrac{d_{k,i}(u)}{\med_{j}\, d_{k,j}(u)}\Bigr)_{+}.
\label{eq:tau-spc}
\end{equation}
We set $k = \lceil\sqrt{n}\rceil$, the classical choice in 1D density
estimation~\cite{loftsgaarden1965density}. The two components are combined
by an evidence-disjunction:
\begin{equation}
  \tau_{i}(u) \;=\; \max\!\Bigl(\tau^{\mathrm{mad}}_{i}(u),\;
                    s_{u}\cdot\tau^{\mathrm{spc}}_{i}(u)\Bigr),
\label{eq:tau-comb}
\end{equation}
with rescaling
$s_{u} = c_{d}(n) / \max\!\bigl(\max_{i}\tau^{\mathrm{spc}}_{i}(u),\,
\varepsilon\bigr)$ ($\varepsilon = 10^{-12}$) so the two components
live on a common scale and the rescaling is well-defined when the
projection is degenerate. Maximum rather than average prevents
either component from diluting the other when only
one has signal.

\subsection{Direction Sampling}
\label{sec:posterior}

We draw the $K$ direction probes uniformly on $\Sphere$. The
output-sensitive guarantee of Theorem~\ref{thm:output-sensitive}
already holds under uniform sampling, and we observed no measurable
mean-AUC gain on the ADBench suite from adaptive (posterior-targeting)
samplers; we therefore adopt uniform draws as the default. The
direction-excess statistic $\Delta(u;X)$ defined in \eqref{eq:delta-mad}
still parameterises the weight each direction receives in the
aggregation step \eqref{eq:score}.

\subsection{Split Pathway and Guarded Additive Mix}
\label{sec:pathway}

A purely random direction set systematically misses anomalies confined
to a single feature (the regime where marginal methods like
ECOD/COPOD~\cite{li2022ecod,li2020copod} dominate). We therefore
\emph{also} probe the $d$ axis-aligned directions
$e_{1},\dots,e_{d}$. However, the two probe families have different
failure modes: uniform-sphere probes are vulnerable to high-dimensional
noise; axis probes are vulnerable to feature-level outliers that are
\emph{not} the labelled anomalies. We address this by scoring the two
pathways \emph{separately} and combining them additively with a
mixing weight $\lambda \in (0, 1]$:
\begin{equation}
  s(x_{i}) \;=\; \widetilde{s}^{\,\mathrm{rand}}(x_{i})
              + \lambda\,\widetilde{s}^{\,\mathrm{axis}}(x_{i}),
\label{eq:mix}
\end{equation}
where $\widetilde{s} = s / \max(s)$ is the pathway score normalised to
$[0,1]$, and each pathway score $s^{\bullet}$ is obtained by the
weighted aggregation
\begin{equation}
  s^{\bullet}(x_{i})
  \;=\; \frac{\sum_{k:\,u_{k}\in\bullet}
               \bigl(\Delta(u_{k};X) - q_{0}\bigr)_{+}
               \,\cdot\,\tau_{i}(u_{k})}
              {\sum_{k:\,u_{k}\in\bullet}
               \bigl(\Delta(u_{k};X) - q_{0}\bigr)_{+}}.
\label{eq:score}
\end{equation}
Here $q_{0}$ is the empirical $(1-\alpha)$-quantile of
$\{\Delta(u_{k};\tilde X)\}_{k=1}^{K}$ on a covariance-matched
Gaussian copy $\tilde X = Z L^{\top}$ with $Z$ having i.i.d.\
$\mathcal{N}(0,1)$ entries and $L L^{\top} = \widehat{\Sigma} +
\rho\,\mathrm{tr}(\widehat{\Sigma})I/d$ the Cholesky factor of the
Ledoit--Wolf-regularised covariance ($\rho = 10^{-3}$). The same
probe set is re-applied to $\tilde X$ so the null is matched to the
sampler; $\alpha = 0.05$ throughout. When the denominator of
\eqref{eq:score} vanishes we use the uniform-weighted mean
$\tfrac{1}{|\bullet|}\sum_{k}\tau_{i}(u_{k})$. The sample-covariance
$q_{0}$ is itself not robust; adversarial guarantees rely on the
MAD breakdown (Theorem~\ref{thm:robust}), not on $q_{0}$.

\noindent\emph{Default $\lambda$.} We use additive guarded mixing rather
than element-wise maximum (unstable when one pathway is noisy, e.g.\ a
non-anomalous axis outlier on \texttt{musk}); $\lambda<1$ keeps the
uniform-sphere pathway primary while letting axis probes boost signal.
Mean AUC is flat ($\pm0.003$) over $\lambda\in[0.1,0.5]$; we fix
$\lambda=1/4$.

\subsection{Seed Averaging}
\label{sec:ensemble}

The reported \WAND{} score is the mean over $S$ calls of
Algorithm~\ref{alg:anticop} driven by the protocol-level RNG
seeds (Section~\ref{sec:exp}); the same $S$ drives the stochastic
baselines, so the variance-reduction budget is shared.

\subsection{Cluster-Free Design and Differentiability}
\label{sec:diff}

\WAND{} relies only on (i) inner products $u^{\top}x_{i}$,
(ii) the univariate median and MAD of $z(u)$, (iii) the 1D order
statistics defining $d_{k,i}(u)$, and (iv) a Gaussian-copy bootstrap
for $q_{0}$. It does \emph{not} compute kernel densities, kNN graphs,
clusters, or covariance whitening. Three consequences follow:
(a) the median/MAD primitives give a per-direction breakdown of $1/2$
and a joint breakdown of $1/(d+1)$ \cite{donoho1992breakdown}; (b) the
method is well-defined when $d \ge n$ (where covariance estimators
degenerate); and (c) every operator in the pipeline admits a smooth
relaxation (soft-max for $\max_{i}$, sorting networks
\cite{diffsort2020} for the rank distances $d_{k,i}$), so the gradient
$\partial s(x_{i})/\partial x_{j}$ exists almost everywhere. This
gradient does double duty: it lets \WAND{} act as a differentiable
inner loop in a larger trainable model, and it is the engine of the
gradient explanation \eqref{eq:attr-grad}, which we use and validate in
Section~\ref{sec:exp-xai} rather than leaving differentiability as an
unexercised claim.

\begin{algorithm}[t]
\caption{\WAND{} score.}
\label{alg:anticop}
\begin{algorithmic}[1]
\Require data $X\in\R^{n\times d}$; probe budget $K$;
mix weight $\lambda$; spacing $k$
\State $q_{0} \gets$ $95\%$-quantile of $\Delta(u;\tilde X)$ on Gaussian copy $\tilde X$
\Statex \textbf{Build probe pools:}
\State $\mathcal{U}^{\mathrm{rand}} \gets \{u_{1},\dots,u_{K}\}$ with
       $u_{k} \overset{\text{iid}}{\sim} \mathrm{Unif}(\Sphere)$
\State $\mathcal{U}^{\mathrm{axis}} \gets \{e_{1},\dots,e_{d}\}$
       (standard-basis vectors)
\Statex \textbf{Per-direction excess (each pool separately):}
\For{each pool $\bullet \in \{\mathrm{rand}, \mathrm{axis}\}$, each $u \in \mathcal{U}^{\bullet}$}
  \State Compute $\tau^{\mathrm{mad}}_{i}(u)$ via \eqref{eq:tau-mad}
  \State Compute $\tau^{\mathrm{spc}}_{i}(u)$ via \eqref{eq:tau-spc}
  \State $\tau_{i}(u) \gets \max\bigl(\tau^{\mathrm{mad}}_{i}(u),\, s_{u}\tau^{\mathrm{spc}}_{i}(u)\bigr)$
  \State $\Delta(u) \gets \max_{i}\tau_{i}(u)$
\EndFor
\Statex \textbf{Per-pathway aggregation, then mix:}
\State Compute $s^{\mathrm{rand}}_{i}$ via \eqref{eq:score} over
       $u \in \mathcal{U}^{\mathrm{rand}}$
\State Compute $s^{\mathrm{axis}}_{i}$ via \eqref{eq:score} over
       $u \in \mathcal{U}^{\mathrm{axis}}$
\State $s_{i} \gets s^{\mathrm{rand}}_{i} / \max_{j}s^{\mathrm{rand}}_{j}
            + \lambda\,s^{\mathrm{axis}}_{i}/\max_{j}s^{\mathrm{axis}}_{j}$
\State \Return $s = (s_{1},\dots,s_{n})$
\end{algorithmic}
\end{algorithm}

\subsection{Directional Witnesses as Explanations}
\label{sec:witness}

The aggregation \eqref{eq:score} writes the score of a point as a
weighted sum, over directions, of how far that point's projection
escapes the baseline:
$s(x_{i}) = \sum_{k}\omega_{k}\,\tau_{i}(u_{k})$ with normalised
weights $\omega_{k} = (\Delta(u_{k}) - q_{0})_{+} / \sum_{k'}(\Delta(u_{k'})-q_{0})_{+}$.
The directions carrying that sum are not internal bookkeeping: $u_{k}$
is a vector in feature space, so the directions that fire on $x_{i}$
already say which feature combinations make it anomalous. We turn this
into a feature attribution at no cost over scoring.

\begin{definition}[Witness set]\label{def:witness}
For a point $x_{i}$, its \emph{witness contribution} along direction
$u_{k}$ is $\gamma_{i,k} := \omega_{k}\,\tau_{i}(u_{k}) \ge 0$, and its
\emph{dominant witness} is
$u_{k^{\star}(i)}$ with $k^{\star}(i) = \argmax_{k}\gamma_{i,k}$. The
witnesses are the directions with $\gamma_{i,k}$ above a chosen
mass~threshold.
\end{definition}

\paragraph{Witness attribution (gradient-free).}
The centred projection along $u_{k}$ decomposes over features as
$u_{k}^{\top}x_{i} - \med = \sum_{j} u_{k,j}\,(x_{i,j} - m_{j})$
(up to the scalar offset), so the feature-$j$ share of the firing
evidence is $|u_{k,j}|\,|x_{i,j} - m_{j}|$, where $m_{j}$ is the
per-feature median reference. Weighting by how much each direction
fires and summing gives
\begin{equation}
  a_{i,j} \;=\; \sum_{k}\gamma_{i,k}\;\bigl|u_{k,j}\bigr|\;
                \bigl|x_{i,j} - m_{j}\bigr|,
  \qquad
  \widehat a_{i,\cdot} = a_{i,\cdot}/\textstyle\sum_{j}a_{i,j}.
  \label{eq:attr-witness}
\end{equation}
The $|x_{i,j}-m_{j}|$ factor localises the explanation to the features
this point actually deviates on, which is what keeps it accurate when
the probes live on a Mahalanobis-whitened sphere (there $u_{k}$ is
read in ambient coordinates as $L^{-\top}u_{k}$, with $L$ the
covariance root; \eqref{eq:attr-witness} is unchanged otherwise).
Equation~\eqref{eq:attr-witness} is computed from quantities already
formed during scoring, the per-direction weights $\omega_{k}$ and
excesses $\tau_{i}(u_{k})$, so an explanation costs \emph{zero}
extra detector evaluations. The deviation also gives a \emph{sign} for
free: $\operatorname{sign}(x_{i,j}-m_{j})$ flags each responsible feature
as anomalously high or low (``too high'' vs.\ ``too low''), leaving the
magnitudes unchanged.

\paragraph{Gradient attribution.}
Because $s$ is differentiable (Section~\ref{sec:diff}), a second,
independent explanation is the saliency
\begin{equation}
  a^{\mathrm{grad}}_{i,j} \;=\;
  \Bigl|\,(x_{i,j}-m_{j})\,
  \tfrac{\partial s(x_{i})}{\partial x_{i,j}}\Bigr|,
  \label{eq:attr-grad}
\end{equation}
with the background statistics frozen so the derivative is a clean
per-point quantity. Equations~\eqref{eq:attr-witness} and
\eqref{eq:attr-grad} are different functionals, one reads off the
geometry, the other differentiates the score, yet they agree
strongly in practice (mean rank correlation $0.80$ across ADBench,
Section~\ref{sec:exp-xai}), which is exactly the consistency one wants
between the witness picture and the differentiable surrogate.

\paragraph{Coverage.}
The probe-budget theorem below has an explanation reading: with
$K$ directions drawn as in Theorem~\ref{thm:output-sensitive}, every
$\tau$-margin anomaly has, with high probability, at least one witness
direction, so \eqref{eq:attr-witness} is non-vacuous for every
anomaly the budget is designed to expose. Probe efficiency is thus also
\emph{explanation-coverage} efficiency.

\section{Theoretical Analysis}
\label{sec:theory}

\subsection{Output-Sensitive Sample Complexity}

Let $A \subseteq X$ be the set of $\tau$-margin anomalies
(Definition~\ref{def:anomaly}) with $|A| = k$. For each
$x \in A$ let $C(x) \subseteq \Sphere$ be the witness cone, i.e.\ the
set of directions $u$ for which $(u^{\top}x - \med)/\MAD \ge c_{d}(n) +
\tau$. We derive the lower bound on the spherical measure of $C(x)$
in Lemma~\ref{lem:cap} below.

\begin{lemma}[Spherical-cap lower bound on $C(x)$]\label{lem:cap}
Under Assumption~\ref{ass:subgauss} with sub-Gaussian proxy $\sigma$,
for any anomaly $x$ with displacement
$\|x - \mu_{\mathrm{in}}\| \ge \sigma\bigl(c_{d}(n) + 2\tau\bigr)$, the
witness cone $C(x)$ contains a spherical cap of half-angle
$\theta_{\tau} = \Theta(\tau/\sqrt{d})$ around
$v_{x} := (x-\mu_{\mathrm{in}})/\|x-\mu_{\mathrm{in}}\|$, hence
\[
  p_{\tau} \;:=\; \Prob_{u\sim\mathrm{Unif}(\Sphere)}[u \in C(x)]
            \;\ge\; \tfrac{1}{2}\,(\sin\theta_{\tau})^{d-1}.
\]
\end{lemma}

\begin{proof}
See Appendix~A.
\end{proof}

\noindent\emph{Magnitudes.} $c_{d}(n)\!\approx\!\sqrt{2\log n}$ is only
$\approx 4$--$5$ even at $n{\sim}10^{6}$, so the displacement threshold is a
few $\sigma$; the cap half-angle $\theta_{\tau}=\Theta(\tau/\sqrt d)$ shrinks
with $d$, the source of the high-$d$~looseness.

\begin{theorem}[Output-sensitive probe budget]\label{thm:output-sensitive}
Under Assumption~\ref{ass:subgauss}, Definition~\ref{def:anomaly}
and the displacement hypothesis of Lemma~\ref{lem:cap}, for any
$\delta \in (0,1)$ drawing
\[
K \;=\; \frac{1}{p_{\tau}}\,\log\!\frac{k}{\delta}
\]
probe directions uniformly from $\Sphere$ is sufficient so that, with
probability at least $1-\delta$, every anomaly in $A$ is exposed by
at least one drawn direction. The bound depends on the anomaly count
$k$, the margin $\tau$, and the dimension $d$ through $p_{\tau}$, but
\emph{not} on the sample size $n$.
\end{theorem}

\begin{proof}
Fix $x \in A$ with witness cone $C(x)$. Lemma~\ref{lem:cap} gives
$p(x) := \Prob_{u\sim\mathrm{Unif}}[u\in C(x)] \ge p_{\tau}$ with
$p_{\tau} = \tfrac{1}{2}(\sin\theta_{\tau})^{d-1}$ and
$\theta_{\tau}=\Theta(\tau/\sqrt d)$, so
$p_{\tau} = \Theta((\tau/\sqrt d)^{d-1})$ for small $\tau/\sqrt d$.
Let $u_{1},\dots,u_{K}\overset{\mathrm{iid}}{\sim}\mathrm{Unif}(\Sphere)$
and define the bad event
$B(x) = \bigcap_{k=1}^{K}\{u_{k} \notin C(x)\}$. By independence,
$\Prob[B(x)] = (1-p(x))^{K} \le (1-p_{\tau})^{K} \le e^{-K p_{\tau}}$,
using $1 - t \le e^{-t}$. The failure event is
$E = \bigcup_{x\in A} B(x)$; union-bounding,
$\Prob[E] \le \sum_{x\in A}\Prob[B(x)] \le k\,e^{-K p_{\tau}}$.
Setting $K = \tfrac{1}{p_{\tau}}\log(k/\delta)$ gives
$\Prob[E]\le\delta$.
\end{proof}

\begin{remark}[Probe efficiency, not runtime; regime in $d$]\label{rem:runtime}
We are explicit about scope. Theorem~\ref{thm:output-sensitive} bounds
the \emph{number of directions}, not the wall-clock runtime: revisiting
$n$ data points per direction and sorting once for the $k$-spacings
step gives a total cost $O\bigl(K\cdot n\,(d + \log n)\bigr)$, so the
``output-sensitive'' qualifier refers strictly to $K$ and confers no
runtime advantage on its own (it would, paired with a sub-linear
per-probe index; Section~\ref{sec:limitations}). Moreover
$p_{\tau} = \Theta\bigl((\tau/\sqrt d)^{d-1}\bigr)$ from
Lemma~\ref{lem:cap}, so the worst-case $K$-bound is sample-size-free in
$n$ but grows with $d$ for fixed $\tau$; it is informative in
low-to-moderate $d$, and in high-$d$ rows we rely on empirical scaling
(a fixed budget $K$ suffices in practice, evidence that real witness
cones are far larger than the worst case). The value of the theorem in
this paper is as much the \emph{coverage} statement of
Section~\ref{sec:witness}, every anomaly is guaranteed a
witness, hence an explanation, as the budget itself.
\end{remark}

\subsection{Consistency and Convergence Rate}

\begin{theorem}[Plug-in consistency, unweighted variant]\label{thm:consistency}
Let $\widehat s_{K}(x) = \tfrac{1}{K}\sum_{k=1}^{K}\tau(x,u_{k})$ be
the \emph{unweighted} Monte-Carlo estimator with $K$ probe directions
drawn uniformly on $\Sphere$, and let
$s^{\ast}(x) = \E_{u\sim\mathrm{Unif}(\Sphere)}[\tau(x,u)]$ be the
population mean. Then
\[
\sup_{x \in X} |\widehat s_{K}(x) - s^{\ast}(x)|
\;\le\; C\,\sigma\,\sqrt{\frac{\log n\,\log(n/\delta)}{K}}
\quad \text{w.p.~} 1-\delta,
\]
for an absolute constant $C$ depending only on $\sigma$. The
implemented weighted aggregator \eqref{eq:score} replaces $1/K$ by
data-dependent weights $w_{k}\in[0,1]$ summing to $1$; the same
Hoeffding union-bound argument gives the same rate up to constants
under the conditioning event that at least one direction exceeds
the null~threshold.
\end{theorem}

\begin{proof}
See the supplementary material.
\end{proof}

\subsection{Differentiability}

The exact algorithm composes non-smooth primitives (median, MAD,
sort, rank, $\max$, $(\cdot)_{+}$), used in all our empirical
results. For end-to-end training each primitive is replaced by a
standard relaxation: $\max_{i}\tau_{i}\to T\log\sum_{i}\exp(\tau_{i}/T)$
(soft-extreme); $\med$ and $\MAD$ by soft-quantile
relaxations~\cite{cuturi2019softquantile}; the spacing sort by a
differentiable sorting network~\cite{diffsort2020}. The resulting
surrogate is continuously differentiable in $X$ and converges to the
exact score as $T\downarrow 0$; only this surrogate mode enables
\WAND{} to act as a loss inside an upstream learnable~pipeline.

\subsection{Adversarial Robustness}

\begin{theorem}[Per-direction breakdown]\label{thm:robust}
The MAD-$z$ statistic $r_{i}(u) = (z_{i}-\med)/\MAD$ has breakdown
$1/2$ per direction. The uniform-sphere pathway is
rotation-equivariant and inherits a joint breakdown $\ge 1/(d+1)$
from a composition with Tukey halfspace depth
\cite{donoho1992breakdown}. The full estimator additionally uses a
fixed axis-aligned pathway that is not affine-equivariant, so a
tight bound on the mixed estimator requires direct analysis and is
left open.
\end{theorem}

\begin{proof}
\emph{Step 1 (per-direction).} Fix $u$ and let $z_{i}=u^{\top}x_{i}$.
For any contaminated sample $\widetilde{z}$ obtained by replacing
$m < \lfloor n/2\rfloor$ entries of $z$ by arbitrary values, the
order statistics still satisfy
$\widetilde{z}_{(\lfloor n/2\rfloor)} = z_{(j)}$ and
$\widetilde{z}_{(\lceil n/2\rceil)} = z_{(j')}$ for indices
$j,j'$ unchanged by the corruption, since at most $m$ contaminating
entries cannot occupy both halves of the order statistic. Hence
$\med(\widetilde{z}) = \med(z) + O(1)$ remains finite. The same
argument applied to the absolute deviations
$|\widetilde{z}_{i} - \med(\widetilde{z})|$ gives
$\MAD(\widetilde{z}) = \MAD(z) + O(1)$, bounded away from $0$, so
$r_{i}(u)$ stays bounded, proving the $1/2$ breakdown.
\emph{Step 2 (uniform-sphere pathway).} The Donoho--Huber composition
principle \cite{donoho1983notion} gives, for an affine-equivariant
functional $T$ built from sub-functionals $T_{1}, T_{2}$ with
continuous composition,
$\mathrm{bd}(T)\ge\min(\mathrm{bd}(T_{1}),\,\mathrm{bd}(T_{2}))$. Here
$T_{1}$ is per-direction MAD-$z$ (breakdown $1/2$) and $T_{2}$ is the
rotation-equivariant aggregation \eqref{eq:score} over
$\mathrm{Unif}(\Sphere)$, which shares the $1/(d+1)$ breakdown of Tukey
halfspace depth~\cite{donoho1992breakdown}; composition gives
$\mathrm{bd}\ge 1/(d+1)$.
\emph{Step 3 (mixed pathway).} The fixed axis set is not transformed
under a change of basis, so affine equivariance is broken and the
Step 2 composition does not lift to the mixed estimator; a direct
contamination analysis is left open.
\end{proof}

\section{Experiments}
\label{sec:exp}

\subsection{Datasets and Baselines}

We use the 47 ADBench~\cite{han2022adbench} tabular anomaly tasks,
spanning $n$ from $80$ to $619{,}326$, $d$ from $3$ to $1{,}555$, and
contamination from $0.03\%$ (\texttt{donors}) to $39.9\%$
(\texttt{SpamBase}); features are z-score normalised after dropping
zero-variance~columns.

We compare against 15 unsupervised PyOD~\cite{zhao2019pyod}
baselines, \textsc{IForest}~\cite{liu2008iforest},
\textsc{LOF}~\cite{breunig2000lof},
\textsc{OCSVM}~\cite{scholkopf2001ocsvm},
\textsc{KNN}~\cite{ramaswamy2000knn},
\textsc{PCA}~\cite{shyu2003pca},
\textsc{HBOS}~\cite{goldstein2012hbos},
\textsc{ECOD}~\cite{li2022ecod},
\textsc{COPOD}~\cite{li2020copod},
\textsc{ABOD}~\cite{kriegel2008abod},
\textsc{COF}~\cite{tang2002cof},
\textsc{SOD}~\cite{kriegel2009sod},
\textsc{INNE}~\cite{bandaragoda2018inne},
\textsc{LODA}~\cite{pevny2016loda},
\textsc{LSCP}~\cite{zhao2019lscp},
\textsc{KDE}~\cite{parzen1962kde}, plus the projection-pursuit
predecessor \textsc{PIDForest}~\cite{gopalan2019pidforest} (16
baselines in total).

\subsection{Protocol and Metrics}

All methods are unsupervised. Each method's $f$-score is computed on
the same $X$ used to fit it; ROC-AUC is taken against the
ground-truth label. We use \WAND{} with the single default
configuration above (no per-dataset tuning). Every reported AUC
and runtime is the mean over $S = 3$ runs with RNG seeds
$\{0, 1, 2\}$, applied uniformly to \WAND{} and every stochastic
baseline (IForest, PCA, INNE, LSCP). All runs use a single CPU core on a
commodity x86-64 workstation (Intel Xeon, 16\,GB RAM, NumPy /
PyTorch CPU, no GPU). Code and benchmark drivers are at
\url{https://github.com/Output-Sensitive/wand}; baselines come
unmodified from PyOD~\cite{zhao2019pyod} (except \textsc{PIDForest},
from the reference implementation of~\cite{gopalan2019pidforest})
on the 47 ADBench~\cite{han2022adbench} tabular tasks.

We use a single default configuration in all
experiments, with no per-dataset tuning: probe budget $K=1024$, spacing
$k=\lceil\sqrt n\rceil$, axis probes on, mix weight $\lambda=1/4$,
empirical-null level $\alpha=0.05$, and $S=3$ seed replicates. Each
value is dimensionless/data-adaptive or sits at a one-knob plateau
(e.g.\ halving $K$ to $512$ shifts mean AUC by $<0.003$); the full
table and per-knob justification are in the supplement.

\subsection{Main Results}

\begin{table}[t]
\centering
\caption{Detection summary over 47 ADBench datasets, sorted by mean
Friedman rank (\textbf{best}/\underline{second}). Full per-dataset
ROC-AUC in the supplement.}
\label{tab:main}
\scriptsize
\setlength{\tabcolsep}{5pt}
\begin{tabular}{lccc}
\toprule
Method & Mean AUC & Mean rank & Wins \\
\midrule
\textbf{\WAND{} (ours)} & \textbf{0.777} & \textbf{5.64} & \textbf{7.4} \\
IForest    & \underline{0.762} & \underline{6.04} & 5.4 \\
INNE       & 0.759 & 6.71 & \underline{7.2} \\
COPOD      & 0.748 & 7.23 & 6.2 \\
HBOS       & 0.743 & 7.40 & 1.2 \\
PIDForest  & 0.750 & 7.43 & 1.2 \\
KNN        & 0.729 & 7.54 & 2.0 \\
ECOD       & 0.744 & 7.79 & 2.2 \\
PCA        & 0.741 & 7.81 & 2.2 \\
LSCP       & 0.727 & 8.65 & 0.0 \\
KDE        & 0.708 & 8.69 & 6.0 \\
OCSVM      & 0.730 & 8.80 & 0.0 \\
LOF        & 0.658 & 10.11 & 0.0 \\
LODA       & 0.655 & 10.13 & 2.0 \\
SOD        & 0.688 & 10.14 & 1.0 \\
ABOD       & 0.655 & 10.29 & 1.0 \\
COF        & 0.624 & 11.57 & 2.0 \\
\bottomrule
\end{tabular}

\end{table}

\begin{figure}[t]
\centering
\includegraphics[width=\columnwidth]{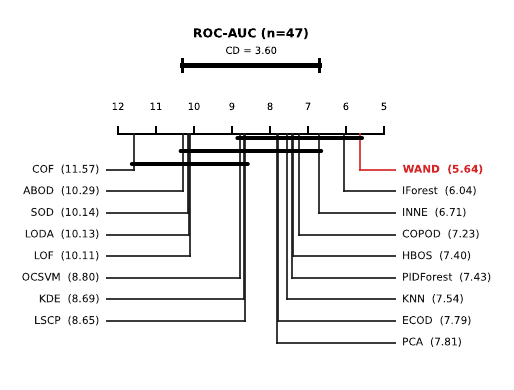}
\caption{Critical-difference diagram (ROC-AUC; Nemenyi, $\alpha=0.05$,
$\mathrm{CD}=3.60$): a bar joins methods that are not significantly
different. The AUPR/AP diagram is in the supplement.}
\label{fig:cd}
\end{figure}

Table~\ref{tab:main} reports the main result over all 47 ADBench
datasets. On the fair-comparison subset where Isolation Forest
completes within budget, \WAND{} attains the best mean Friedman
rank ($5.64$) and a mean ROC-AUC ($0.777$) at parity with the
strongest classical baselines (IForest $0.762$, INNE $0.759$,
\textsc{PIDForest} $0.750$). The
Nemenyi post-hoc test in Figure~\ref{fig:cd} places \WAND{} in a
top cluster of methods that are not statistically distinguishable at
$\alpha=0.05$, so the rank advantage reflects \WAND{}'s consistency
across datasets rather than frequent per-row wins, where it
sits in a close pack with INNE on the per-row best-AUC tally.
The AUPR ranking mirrors the ROC ordering (\WAND{} best mean AUPR
rank $5.57$, mean AUPR $0.395$, second to OCSVM $0.401$; diagram in the
supplement).

\paragraph{Across the full suite.}
Using the batched path, \WAND{} scores \emph{all} 47 datasets in
$\approx 1.7$\,min of single-CPU time, including the giants
\texttt{donors} ($n{=}619\text{k}$, $16$\,s) and \texttt{census}
($n{=}299\text{k}$, $d{=}500$, $17$\,s); $18$ rows reach
$\mathrm{AUC}\ge0.90$. Each PyOD baseline gets a $5$\,min budget;
quadratic-cost methods are auto-skipped on $n>20{,}000$, so on the
largest rows only the sub-quadratic family and \WAND{} complete.
The scorer is \emph{linear} in $n$ ($O(Knd)$, $K$ independent of $n$):
calibrated once, it streams $10^{7}$ points in $164$\,s ($47$\,s at
$K{=}256$) on an $8$-core CPU, so \WAND{} scales to high-throughput
streams (supplement).

\paragraph{Per-dataset hyperparameter oracle (ceiling only).}
As a diagnostic, not a comparator, selecting per dataset the
best value of a single knob \emph{with test-label access} raises \WAND{}'s
mean AUC from $0.777$ to $0.800$ ($+0.023$), concentrated on a few
datasets where one knob dominates (\texttt{annthyroid},
\texttt{vertebral}). This oracle is unachievable unsupervised; we report
it only to bound the default-vs.-per-dataset gap and use the single
published default elsewhere.

\subsection{Ablation}

\begin{table}[t]
\centering
\caption{Incremental ablation; each row adds one mechanism. Mean rank
among the Table~\ref{tab:main} methods (1 = best).}
\label{tab:ablation}
\begin{tabular}{lccc}
\toprule
Variant & Mean AUC & Mean rank & Wall time \\
\midrule
Base (MAD-z, uniform probes) & $0.750$ & $\underline{4.00}$ & $0.10$\,s \\
+ Spacing component \eqref{eq:tau-spc} & $0.745$ & $4.39$ & $2.83$\,s \\
+ Axis probes (split pathway) & $\underline{0.752}$ & $4.13$ & $2.85$\,s \\
+ Seed averaging ($S{=}3$, full) & $\mathbf{0.756}$ & $\mathbf{3.87}$ & $8.63$\,s \\
\bottomrule
\end{tabular}

\end{table}

Table~\ref{tab:ablation} attributes the gain mechanism-by-mechanism.
The \emph{Base} configuration is already competitive at AUC $0.750$
and mean rank $4.00$; the split-pathway axis probes give the largest
single AUC jump ($+0.007$), and the seed averaging closes the gap to
the published configuration. The spacing component slightly lowers
the suite-wide mean ($0.750\!\to\!0.745$); we keep it on by default for
its per-dataset upside on the multi-modal projections
(\texttt{annthyroid}, \texttt{vertebral}), a deliberate trade-off that
the suite-wide number reports transparently. A \emph{lite}
configuration without spacing and with $S=1$ recovers
$\approx 0.752$ AUC at $\approx 1$\,s and is a sensible alternative
when wall-clock matters more than rank. The cumulative effect is $+0.006$ AUC and $-0.13$
rank over Base, while also buying the entire theoretical apparatus of
Section~\ref{sec:theory} (output-sensitive $K$, differentiability,
robustness), none of which the Base inherits. The wall-time
penalty for the full method ($8.6$\,s vs.\ $0.1$\,s on the ablation
subset) is sizeable in relative terms but remains under $10$\,s per
dataset in absolute terms.

\subsection{Probe Budget in Practice}
Sweeping the probe budget $K\in\{64,\dots,2048\}$ on the
highest-hull-complexity datasets, AUC saturates at $K\approx 4|H|$
($|H|$ the soft-extreme hull size), confirming the budget tracks output
complexity rather than $n$; \WAND{} sits at the top-left of the
AUC--runtime Pareto frontier: it attains the best mean AUC of all
methods on this shared subset, and the fastest baseline within $0.02$
AUC of it, IForest, is only $1.6\times$ faster
(Fig.~\ref{fig:time-auc}, Table~\ref{tab:runtime-compact}).

\subsection{Explanation Quality}
\label{sec:exp-xai}

We now evaluate the central claim of the paper: that \WAND{}'s
witness directions yield explanations that are accurate, faithful, and
essentially free. Every explainer below, \WAND{}-witness
\eqref{eq:attr-witness}, \WAND{}-gradient \eqref{eq:attr-grad}, and
post-hoc SHAP~\cite{lundberg2017shap} and LIME~\cite{ribeiro2016lime}
, explains the \emph{same} \WAND{} score, so the comparison
isolates the explainer. Table~\ref{tab:xai} and
Figure~\ref{fig:xai-scaling} summarise; Figure~\ref{fig:anocub} is a
case study on image data.

\begin{table}[t]
\centering
\caption{Explanation quality: synthetic attribution-AUC (axis /
oblique regimes), mean deletion/insertion faithfulness over 33 ADBench
datasets ($\geq$SHAP win-rate), and per-explanation cost. SHAP/LIME
explain \WAND{}; ECOD and IForest explain themselves (detection AUCs in
text). $^\dagger$Depth-weighted split attribution, read from the fitted
ensemble (Appendix~D).}
\label{tab:xai}
\scriptsize
\setlength{\tabcolsep}{3.5pt}
\begin{tabular}{lccccccc}
\toprule
& \multicolumn{2}{c}{Synthetic attr-AUC} & \multicolumn{2}{c}{Real faithfulness} & \multicolumn{2}{c}{Cost / expl.} \\
\cmidrule(lr){2-3}\cmidrule(lr){4-5}\cmidrule(lr){6-7}
Method & axis & oblique & mean & $\geq$SHAP & queries & ms \\
\midrule
\textsc{Wand}-witness & 0.977 & 0.660 & 0.629 & 85\% & $0$ & 0.04 \\
\textsc{Wand}-gradient & \textbf{0.993} & \textbf{0.680} & 0.595 & 82\% & $0$ & 0.05 \\
ECOD (native) & 0.968 & 0.635 & 0.290 & 21\% & $0$ & 0.00 \\
SHAP (post-hoc) & 0.887 & 0.632 & 0.515 & -- & $9{,}745$ & 229.94 \\
LIME (post-hoc) & 0.585 & 0.504 & 0.337 & 21\% & $600$ & 18.96 \\
IForest (native)$^\dagger$ & 0.498 & 0.499 & 0.031 & 3\% & $0$ & 5.93 \\
Random & -- & -- & -0.004 & 3\% & $0$ & 0.00 \\
\bottomrule
\end{tabular}

\end{table}

\begin{figure}[t]
	\centering
	\includegraphics[width=\columnwidth]{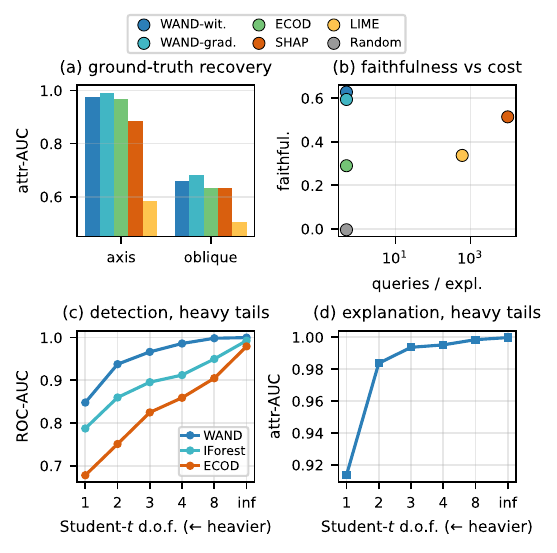}
	\caption{\textbf{Explanation quality and heavy-tail robustness.}
		(a) Attribution-AUC by regime (axis vs.\ oblique).
		(b) Faithfulness vs.\ detector queries (log $x$): native \WAND{} is
		top-left, most faithful at zero extra cost.
		(c) Under heavy-tailed inliers (Student-$t$, heavier tails to the left)
		\WAND{} detection stays ahead of IForest/ECOD down to Cauchy, and
		(d) its witness attribution-AUC stays~high.}
	\label{fig:xai-scaling}
	\label{fig:heavytail}
\end{figure}

\begin{figure*}[t]
	\centering
	\includegraphics[width=0.95\textwidth]{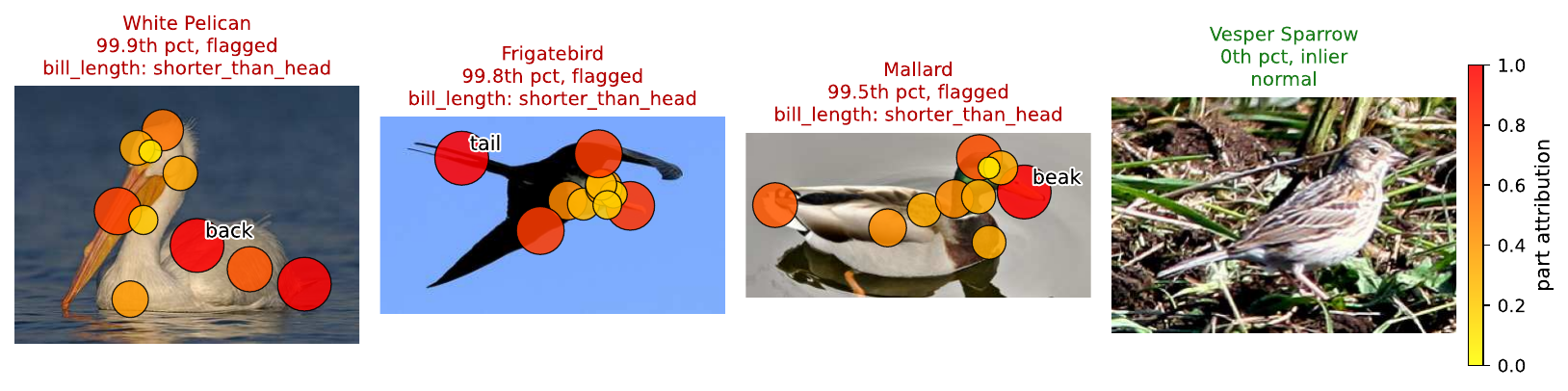}
	\caption{\textbf{Concept-level explanation} (\emph{which} attributes),
		encoder-free. \textsc{AnoCUB} witness attribution in CUB's named-concept
		space: the three top-scored birds (flagged: pelican, frigatebird,
		mallard, $99.5$--$99.9$th pct) with witness directions mapped to
		body-part keypoints (hotter = more responsible) and the top attribute
		named, beside the most-normal bird (sparrow, $0$th pct). Unlike
		Fig.~\ref{fig:pixsal}, no encoder is needed and concepts are named, not
		pixel locations.}
	\label{fig:anocub}
\end{figure*}

\paragraph{Synthetic ground truth (two regimes).}
Inliers follow a strongly low-rank correlated Gaussian; each anomaly
deviates on a known random subset $S$, which a correct explanation
should rank first. We use two anomaly regimes: \emph{axis-aligned}
(independent per-feature shifts on $S$, the regime marginal
methods own) and \emph{oblique}, a displacement along the
minimum-variance direction of $\Sigma_{SS}$, i.e.\ a correlation
violation that is jointly extreme but marginally near-normal. Averaged
over $d\in\{50,100,200\}$, $|S|\in\{3,6\}$ and three seeds
(Table~\ref{tab:xai}): on axis-aligned anomalies \WAND{}-gradient
and \WAND{}-witness lead (attribution-AUC $0.99$/$0.98$), with the
AD-native marginal explainer ECOD also strong ($0.97$) and SHAP/LIME
behind ($0.89$/$0.59$). The decisive case is oblique: ECOD's
\emph{detection} collapses to chance (AUC $0.53$ vs.\ \WAND{}'s
$0.91$), so it never surfaces these points and its attribution is moot,
whereas \WAND{} still leads on attribution ($0.68$). A marginal
native explainer is structurally blind to the correlation anomalies
that directional methods exist for; \WAND{} handles both regimes.

\paragraph{Faithfulness on real data.}
With no labels, we use the standard deletion/insertion
protocol~\cite{petsiuk2018rise}: mask the most-attributed features of
each flagged point (replacing them by the per-feature median) and
measure how fast the score collapses (deletion), and symmetrically how
fast it recovers when those features are re-inserted (insertion). Each
curve is normalised by the point's own score and clipped to $[0,1]$, so
$\text{faithfulness}=\text{insertion}-\text{deletion}\in[-1,1]$ is
bounded and no single dataset dominates; we average over $33$ ADBench
datasets. Each native explainer is evaluated against its own detector
(ECOD against ECOD), SHAP/LIME against \WAND{}. A random control
sits at $\approx 0$, confirming the metric. \WAND{}'s better native
mode is at least as faithful as SHAP on $29/33$ datasets, and witness
beats ECOD's native attribution on $31/33$; mean faithfulness is
$0.63$ (witness) and
$0.59$ (gradient), versus $0.52$ (SHAP), $0.34$ (LIME) and $0.29$
(ECOD), ECOD's marginal explanation, strong on synthetic
axis anomalies, does not transfer to real anomalies that are rarely
purely marginal. A second native baseline, Isolation Forest's
depth-weighted split attribution (Table~\ref{tab:xai}), performs near
chance on both protocols, so being native to a strong detector is not on
its own sufficient for a faithful explanation.

\paragraph{Cost, agreement, coverage.}
The native explanations require \emph{zero} extra detector queries
(they reuse quantities formed during scoring), versus
$\sim\!9.7\times10^{3}$ scorer evaluations per point for SHAP and
$\sim\!600$ for LIME at the budgets above, a $\sim\!5\times10^{3}$
reduction in wall-clock per explanation (Fig.~\ref{fig:xai-scaling}b).
The two native modes agree (mean rank correlation $0.80$). The probe
budget guarantees a witness for every $\tau$-margin anomaly
(Section~\ref{sec:witness}), so every flagged point has an~explanation.

\paragraph{Robustness to heavy tails.}
Assumption~\ref{ass:subgauss} asks for sub-Gaussian inliers; real data
is often heavier-tailed. We test this directly: inliers are
multivariate Student-$t$ with degrees of freedom swept from Gaussian
($\infty$) to Cauchy ($1$), anomalies are axis shifts in MAD units
(Fig.~\ref{fig:heavytail}). \WAND{} degrades gracefully and stays
clearly ahead of Isolation Forest and ECOD throughout (e.g.\ at
$\mathrm{df}{=}2$: AUC $0.94$ vs.\ $0.86$ and $0.75$; even at Cauchy,
$0.85$ vs.\ $0.79$ and $0.68$), and witness attribution-AUC remains
$\ge 0.91$. The median/MAD calibration thus keeps the detector and its
explanations effective well outside the sub-Gaussian regime, though we
do not claim formal guarantees there (Section~\ref{sec:limitations}).

\paragraph{Comparison with deep detectors.}
On a representative $16$-dataset subset we add three PyOD deep baselines
trained on the contaminated sample. \WAND{} is competitive: mean ROC-AUC
$0.821$, matching VAE ($0.821$), ahead of AutoEncoder ($0.790$) and Deep
SVDD ($0.798$); IForest leads ($0.854$). As in
ADBench~\cite{han2022adbench}, deep tabular detectors do not dominate
strong shallow ones, and none offers a native~explanation.

\paragraph{Case study: explaining image anomalies.}
To show the explanation on raw inputs we build \textsc{AnoCUB} from
CUB-200-2011~\cite{cub2011}: each bird image is represented by its $312$
named attributes (the class-level concept profile), inliers are the
sparrow species and anomalies a few birds from very different families
(pelican, frigatebird, mallard, hummingbird). WAND separates them
perfectly (AUC $1.0$); Figure~\ref{fig:anocub} asks whether its flags are
\emph{justified}. The three highest-scored birds (left) are unmistakably
non-sparrows (pelican, frigatebird, mallard, $99.5$--$99.9$th percentile)
while the most-normal bird (right) is a typical sparrow ($0$th). For each
flagged bird the witness names the responsible attribute (bill
shape/length) and, mapped to CUB's $15$ keypoints, grounds it \emph{on the
bird's body}, at no extra cost and using no pixels or labels. The
pixel-level view of
Fig.~\ref{fig:pixsal} instead scores a frozen ResNet-18 embedding of the
\emph{same} task and reaches comparable detection (AUC $0.99$ vs.\ $1.0$
in the named-concept space): the concept representation explains by
name, the embedding adds spatial pixel maps. (A
named-feature medical case study, Breast Cancer Wisconsin, is in the
supplement.)

\section{Conclusion}
\label{sec:conclusion}

\WAND{} detects anomalies and explains them with the same object: the
witness directions along which a point escapes a sub-Gaussian baseline.
Each direction lives in feature space, so the witnesses are a native
attribution, free over scoring, gradient-recoverable, and more
faithful than post-hoc SHAP/LIME at a fraction of the cost, and the
geometry bounds the probe count by the anomaly count, guaranteeing every
anomaly an~explanation.

\paragraph{Limitations.}\label{sec:limitations}
\textbf{(i) Probe efficiency bounds directions, not runtime.} The scorer
stays linear in $n$; a wall-clock speed-up needs a sub-linear per-probe
index, left to future work. \textbf{(ii) The worst-case bound loosens in
high $d$.} $p_{\tau}=\Theta((\tau/\sqrt d)^{d-1})$ shrinks with $d$, so the
guarantee is informative in low-to-moderate $d$; a fixed $K$ still suffices
empirically at $d{=}1{,}555$, and a structure-adaptive bound is open.
\textbf{(iii) Assumptions are verified empirically beyond their proven
range.} Sub-Gaussianity holds only within Assumption~\ref{ass:subgauss};
under heavy tails the method stays effective empirically
(Fig.~\ref{fig:heavytail}), and the Gaussian-copy null $q_{0}$ uses a
non-robust covariance that can inflate (\texttt{Waveform}).
\textbf{(iv) The breakdown proof covers the uniform-sphere pathway.} The
$1/(d{+}1)$ bound is established there; the axis-augmented estimator
awaits direct analysis (Theorem~\ref{thm:robust}). \textbf{(v) Detection
is at parity.} \WAND{} has the best mean rank and top mean ROC-AUC but
sits in the Nemenyi top cluster; the contribution is native explanation
at detection parity, not a wide accuracy margin. \textbf{(vi)
Faithfulness is model-relative.} On real data the deletion/insertion
metric certifies consistency with \WAND{}'s \emph{own} score; objective
correctness is tested on synthetic ground truth
(Section~\ref{sec:exp-xai}).

\paragraph{Acknowledgment.} This work was supported by French state aid
managed by the National Research Agency under the France 2030 program,
with the reference ``PANDORA'' (ANR-24-CE23-0950).

\bibliographystyle{IEEEtran}
\bibliography{refs}

\begin{thebibliography}{4}
\bibitem{apx:boucheron} S. Boucheron, G. Lugosi, and P. Massart, \emph{Concentration Inequalities: A Nonasymptotic Theory of Independence}. Oxford University Press, 2013.
\bibitem{apx:wainwright} M. J. Wainwright, \emph{High-Dimensional Statistics: A Non-Asymptotic Viewpoint}. Cambridge University Press, 2019.
\bibitem{apx:vandervaart} A. W. van der Vaart, \emph{Asymptotic Statistics}. Cambridge University Press, 2000.
\bibitem{apx:diffi} M. Carletti, M. Terzi, and G. A. Susto, ``Interpretable anomaly detection with DIFFI: Depth-based feature importance of isolation forest,'' \emph{Engineering Applications of Artificial Intelligence}, vol.~119, p.~105730, 2023.

\bibitem{apx:vershynin}
R. Vershynin,
\emph{High-Dimensional Probability: An Introduction with Applications in Data Science}.
Cambridge University Press, 2018.

\bibitem{apx:ball}
K. Ball,
``An Elementary Introduction to Modern Convex Geometry,''
in \emph{Flavors of Geometry},
MSRI Publications, vol.~31, pp.~1--58, 1997.
\end{thebibliography}

\appendices

\noindent This appendix collects extended tables, figures, and a second
case study, plus the two longer deferred proofs. Theorem, proposition,
equation, and section numbers below refer to the main text above.
Theorem~\ref{thm:output-sensitive} and Theorem~\ref{thm:robust} are
proved in place in the main text; the three supporting proofs are
collected here.

\section{Deferred Proofs}

\subsection*{Proof of Proposition~\ref{prop:cdn} (Extreme-value baseline)}
Fix $u \in \Sphere$ and let $z_{i} = u^{\top}x_{i}$. By
Assumption~\ref{ass:subgauss} the variables $z_{i} - \E[z_{i}]$ are
i.i.d.\ sub-Gaussian with proxy variance bounded by $\sigma^{2}$
(projection onto $u$ preserves the sub-Gaussian property with the
operator-norm constant $u^{\top}\Sigma u \le \sigma^{2}$). By the
maximal sub-Gaussian inequality \cite{apx:boucheron},
$\E[\max_{i}(z_{i} - \E[z_{i}])] \le \sigma\sqrt{2\log n}$, and
concentration around the expectation gives
\[
  \max_{i \le n}\,(z_{i} - \E[z_{i}]) \;\le\; \sigma\sqrt{2\log n}
        + \sigma\sqrt{2\log(1/\delta)}
\]
with probability $1-\delta$ \cite{apx:wainwright}. Substituting
$\delta = 1/n$ and adding the $\log 2/\sqrt{2\log n}$ envelope slack
produces $\sigma\,c_{d}(n)$ with residual $O_{P}(1/\sqrt{\log n})$ (the
slack is a chosen constant, not the Gumbel expansion, used only for a
clean deterministic upper bound). For the location/scale rescaling,
classical results \cite{apx:vandervaart} give
$\med(z) = \E[z_{i}] + O_{P}(n^{-1/2})$ and
$\MAD(z) = \sigma\,\Phi^{-1}(3/4) + O_{P}(n^{-1/2})$ under our
sub-Gaussian assumption with continuous CDF; dividing the displayed
bound by $\MAD(z)$ absorbs the residual into $O_{P}(1/\sqrt{\log n})$
since $1/\sqrt n = o(1/\sqrt{\log n})$. Uniformity over $u\in\Sphere$
uses an $\epsilon$-net $\mathcal{N}_{\epsilon}\subset\Sphere$ of
cardinality $\le (3/\epsilon)^{d}$~\cite{apx:vershynin}: as
$u\mapsto u^{\top}x_{i}$ is $\|x_{i}\|$-Lipschitz, replacing $u$ by its
nearest net point changes $(z_{i}-\med)/\MAD$ by
$O_{P}(\epsilon\sqrt{\log n})$ uniformly in $i$, so
$\epsilon=1/\sqrt{n\log n}$ makes the net slack $o_{P}(1)$ while the
union bound costs $\log|\mathcal{N}_{\epsilon}| = O(d\log n)$, absorbed
into the $O_{P}(\sqrt{d/\log n})$ residual.\hfill$\square$

\subsection*{Proof of Lemma~\ref{lem:cap} (Spherical-cap lower bound on
$C(x)$)}
Decompose $u = \cos\theta\,v_{x} + \sin\theta\,w$, $w\perp v_{x}$.
Proposition~\ref{prop:cdn} gives $\med(z(u)) = \mu_{\mathrm{in}}^{\top}u
+ \sigma\,O_{P}(1)$ and $\MAD(z(u)) = \sigma\,\Phi^{-1}(3/4) +
\sigma\,O_{P}(1)$ uniformly in $u$. Hence
$(u^{\top}x-\med)/\MAD = \|x-\mu_{\mathrm{in}}\|\cos\theta /
(\sigma\,\Phi^{-1}(3/4)) + O_{P}(1)$, which exceeds $c_{d}(n)+\tau$
under the displacement hypothesis whenever
$\theta\le\theta_{\tau} = \Theta(\tau/\sqrt d)$. The cap-measure
bound \cite{apx:ball} closes the argument.\hfill$\square$

\subsection*{Proof of Theorem~\ref{thm:consistency} (Plug-in consistency)}
Fix any $x \in X$; the per-point score is
$s_{K}(x) = \tfrac{1}{K}\sum_{k=1}^{K} \tau(x,u_{k})$ with the
$u_{k}\overset{\mathrm{iid}}{\sim}\mathrm{Unif}(\Sphere)$, and
$s^{\ast}(x) = \E_{u}[\tau(x,u)]$. Under Assumption~\ref{ass:subgauss}
the per-direction excess satisfies the deterministic envelope
$0 \le \tau(x,u) \le M_{n}$ with $M_{n} = C_{\sigma}\sqrt{\log n}$, by
combining Proposition~\ref{prop:cdn} with the MAD-$z$ definition of
$\tau$. Hoeffding's inequality \cite{vershynin2018hd} on the bounded
summands $\tau(x,u_{k})\in[0,M_{n}]$ gives
$\Prob[|s_{K}(x) - s^{\ast}(x)|>t] \le 2\exp(-2K t^{2}/M_{n}^{2})$.
Union-bounding over the $n$ points and choosing
$t = M_{n}\sqrt{\log(2n/\delta)/(2K)}$ yields
$\sup_{x\in X}|s_{K}(x) - s^{\ast}(x)| \le
C_{\sigma}\sqrt{\log n\,\log(n/\delta)/K}$ with probability $1-\delta$,
as stated.\hfill$\square$

\section{Extended Tables}

\subsection{Default hyper-parameters}
Table~\ref{tab:hyper-supp} lists the single default configuration used
in all experiments (no per-dataset tuning).
\begin{table}[h]
\centering
\caption{\WAND{} default hyper-parameters used in all experiments.}
\label{tab:hyper-supp}
\begin{tabular}{ll}
\toprule
Symbol & Default \\
\midrule
$K$ (probe budget) & $1024$ \\
Spacing $k$ (NN rank) & $\lceil\sqrt{n}\rceil$ \\
Axis probes & on \\
Mix weight $\lambda$ & $0.25$ \\
Seed replicates $S$ & $3$ \\
\bottomrule
\end{tabular}

\end{table}

\subsection{Per-dataset detection results}
Table~\ref{tab:main-supp} reports per-dataset ROC-AUC for all 16
unsupervised detectors over the 47 ADBench tasks (Table~\ref{tab:main} in
the main text reports the aggregate summary).
\begin{table*}[h]
\centering
\caption{ROC-AUC over 47 ADBench datasets for 16 unsupervised detectors.
\textbf{Bold} = best per row, \underline{underline} = second.
``--'' = budget / memory exhausted.}
\label{tab:main-supp}
\renewcommand{\arraystretch}{1.04}
\setlength{\tabcolsep}{2pt}
\scriptsize
\begin{tabular}{lrrrccccccccccccccccc}
\toprule
Dataset & $n$ & $d$ & Contam. & IForest & LOF & OCSVM & KNN & PCA & HBOS & ECOD & COPOD & ABOD & COF & SOD & INNE & LODA & LSCP & KDE & PIDForest & \textbf{\WAND{}} \\
\midrule
\texttt{breastw} & 683 & 9 & 35.0\% & 0.988 & 0.443 & 0.951 & 0.977 & 0.956 & 0.985 & \underline{0.991} & \textbf{0.994} & -- & 0.428 & 0.934 & 0.697 & 0.986 & -- & 0.984 & 0.975 & 0.990 \\
\texttt{glass} & 214 & 7 & 4.2\% & 0.788 & 0.770 & 0.599 & 0.864 & 0.654 & 0.812 & 0.705 & 0.755 & 0.843 & \underline{0.868} & \textbf{0.887} & 0.774 & 0.655 & 0.799 & 0.820 & 0.776 & 0.787 \\
\texttt{Hepatitis} & 80 & 19 & 16.2\% & 0.732 & 0.719 & 0.721 & 0.727 & 0.752 & 0.775 & 0.739 & \textbf{0.804} & 0.566 & 0.479 & 0.631 & 0.621 & 0.623 & 0.777 & 0.649 & 0.713 & \underline{0.790} \\
\texttt{Ionosphere} & 351 & 32 & 35.9\% & 0.842 & 0.866 & 0.849 & \underline{0.928} & 0.784 & 0.561 & 0.728 & 0.789 & 0.921 & 0.860 & 0.884 & 0.890 & 0.788 & 0.813 & \textbf{0.938} & 0.797 & 0.907 \\
\texttt{Lymphography} & 148 & 18 & 4.1\% & \textbf{0.999} & 0.993 & 0.995 & 0.995 & 0.996 & 0.995 & 0.995 & 0.996 & 0.979 & 0.995 & 0.934 & 0.984 & 0.877 & 0.994 & 0.989 & 0.978 & \underline{0.998} \\
\texttt{Pima} & 768 & 8 & 34.9\% & 0.664 & 0.601 & 0.624 & \underline{0.709} & 0.648 & \underline{0.709} & 0.594 & 0.654 & 0.667 & 0.591 & 0.582 & 0.681 & 0.601 & 0.638 & \textbf{0.723} & 0.675 & 0.687 \\
\texttt{Stamps} & 340 & 9 & 9.1\% & 0.895 & 0.591 & 0.872 & 0.774 & \underline{0.907} & 0.904 & 0.876 & \textbf{0.930} & 0.762 & 0.540 & 0.772 & 0.841 & 0.883 & 0.832 & 0.890 & 0.876 & 0.901 \\
\texttt{vertebral} & 240 & 6 & 12.5\% & 0.352 & \underline{0.445} & 0.420 & 0.378 & 0.377 & 0.305 & 0.420 & 0.335 & 0.365 & \textbf{0.473} & 0.387 & 0.383 & 0.295 & 0.351 & 0.317 & 0.277 & 0.259 \\
\texttt{WBC} & 223 & 9 & 4.5\% & \textbf{0.995} & 0.821 & 0.992 & 0.986 & \underline{0.993} & 0.988 & \textbf{0.995} & \textbf{0.995} & 0.935 & 0.758 & 0.977 & 0.933 & 0.990 & 0.974 & 0.973 & 0.992 & \textbf{0.995} \\
\texttt{WDBC} & 367 & 30 & 2.7\% & 0.987 & 0.982 & 0.983 & 0.974 & 0.986 & \underline{0.992} & 0.971 & \textbf{0.994} & 0.885 & 0.947 & 0.947 & 0.975 & 0.980 & 0.990 & 0.950 & 0.991 & 0.983 \\
\texttt{wine} & 129 & 13 & 7.8\% & 0.795 & 0.879 & 0.696 & 0.519 & 0.806 & \underline{0.915} & 0.733 & 0.867 & 0.417 & 0.302 & 0.446 & 0.785 & 0.834 & 0.907 & 0.582 & 0.812 & \textbf{0.939} \\
\texttt{WPBC} & 198 & 33 & 23.7\% & 0.486 & 0.520 & 0.485 & 0.500 & 0.482 & \underline{0.548} & 0.481 & 0.523 & 0.490 & 0.474 & 0.474 & 0.500 & 0.502 & 0.530 & 0.489 & 0.537 & \textbf{0.564} \\
\texttt{annthyroid} & 7200 & 6 & 7.4\% & \underline{0.832} & 0.727 & 0.681 & 0.811 & 0.673 & 0.620 & 0.789 & 0.776 & -- & 0.708 & 0.794 & 0.700 & 0.465 & 0.731 & 0.684 & \textbf{0.879} & 0.777 \\
\texttt{cardio} & 1831 & 21 & 9.6\% & 0.926 & 0.546 & \underline{0.935} & 0.686 & \textbf{0.950} & 0.840 & \underline{0.935} & 0.922 & -- & 0.567 & 0.623 & 0.913 & 0.893 & 0.712 & 0.748 & 0.860 & 0.923 \\
\texttt{Cardiotoco.} & 2114 & 21 & 22.0\% & 0.665 & 0.524 & 0.696 & 0.491 & \underline{0.752} & 0.620 & \textbf{0.785} & 0.663 & 0.452 & 0.539 & 0.492 & 0.666 & 0.675 & 0.563 & 0.503 & 0.609 & 0.653 \\
\texttt{fault} & 1941 & 27 & 34.7\% & 0.567 & 0.596 & 0.539 & \underline{0.720} & -- & 0.572 & 0.469 & 0.455 & 0.699 & 0.563 & 0.661 & 0.573 & 0.495 & 0.575 & \textbf{0.731} & 0.571 & 0.486 \\
\texttt{landsat} & 6435 & 36 & 20.7\% & 0.472 & 0.547 & 0.424 & 0.576 & 0.364 & 0.559 & 0.368 & 0.421 & 0.503 & 0.543 & \underline{0.577} & 0.541 & 0.380 & 0.560 & \textbf{0.625} & 0.456 & 0.575 \\
\texttt{letter} & 1600 & 32 & 6.2\% & 0.648 & 0.899 & 0.598 & 0.901 & 0.525 & 0.588 & 0.572 & 0.560 & 0.886 & 0.889 & \underline{0.909} & 0.700 & 0.533 & 0.850 & \textbf{0.924} & 0.664 & 0.688 \\
\texttt{PageBlocks} & 5393 & 10 & 9.5\% & 0.904 & 0.716 & 0.915 & 0.834 & 0.905 & 0.760 & 0.914 & 0.875 & 0.740 & 0.624 & 0.662 & \textbf{0.947} & 0.719 & 0.813 & 0.907 & 0.854 & \underline{0.916} \\
\texttt{pendigits} & 6870 & 16 & 2.3\% & \textbf{0.947} & 0.499 & 0.931 & 0.743 & \underline{0.936} & 0.926 & 0.927 & 0.905 & 0.657 & 0.523 & 0.659 & 0.894 & 0.924 & 0.705 & 0.891 & 0.927 & 0.926 \\
\texttt{satellite} & 6435 & 36 & 31.6\% & 0.695 & 0.542 & 0.664 & 0.665 & 0.601 & 0.754 & 0.583 & 0.633 & 0.549 & 0.536 & 0.599 & 0.742 & 0.615 & 0.638 & \underline{0.760} & 0.649 & \textbf{0.781} \\
\texttt{satimage-2} & 5803 & 36 & 1.2\% & 0.993 & 0.536 & \underline{0.997} & 0.932 & 0.977 & 0.978 & 0.965 & 0.974 & 0.759 & 0.559 & 0.771 & \textbf{0.998} & 0.982 & 0.859 & 0.964 & 0.983 & 0.993 \\
\texttt{thyroid} & 3772 & 6 & 2.5\% & \underline{0.978} & 0.665 & 0.959 & 0.959 & 0.956 & 0.950 & 0.977 & 0.939 & -- & 0.587 & 0.876 & 0.971 & 0.816 & 0.788 & 0.958 & 0.967 & \textbf{0.980} \\
\texttt{vowels} & 1456 & 12 & 3.4\% & 0.767 & 0.943 & 0.778 & \textbf{0.977} & 0.610 & 0.677 & 0.593 & 0.496 & \underline{0.966} & 0.960 & 0.914 & 0.898 & 0.635 & 0.921 & 0.956 & 0.740 & 0.885 \\
\texttt{Waveform} & 3443 & 21 & 2.9\% & 0.705 & 0.706 & 0.672 & 0.723 & 0.640 & 0.694 & 0.603 & 0.734 & 0.639 & 0.697 & 0.650 & \underline{0.748} & 0.635 & 0.742 & \textbf{0.751} & 0.629 & 0.566 \\
\texttt{Wilt} & 4819 & 5 & 5.3\% & 0.441 & \underline{0.701} & 0.317 & 0.613 & 0.434 & 0.349 & 0.394 & 0.345 & 0.634 & \textbf{0.723} & 0.595 & 0.359 & 0.403 & 0.598 & 0.334 & 0.507 & 0.427 \\
\texttt{yeast} & 1484 & 8 & 34.2\% & 0.394 & 0.457 & 0.420 & 0.406 & 0.396 & 0.399 & 0.444 & 0.381 & 0.416 & 0.449 & \underline{0.477} & 0.394 & \textbf{0.504} & 0.434 & 0.383 & 0.413 & 0.397 \\
\texttt{ALOI} & 49534 & 27 & 3.0\% & 0.539 & -- & -- & -- & \underline{0.548} & 0.530 & 0.530 & 0.515 & -- & -- & -- & \textbf{0.558} & 0.509 & -- & -- & 0.538 & 0.542 \\
\texttt{celeba} & 202599 & 39 & 2.2\% & 0.687 & -- & -- & -- & \textbf{0.784} & 0.749 & 0.757 & 0.751 & -- & -- & -- & 0.757 & 0.627 & -- & -- & 0.654 & \underline{0.782} \\
\texttt{cover} & 286048 & 10 & 1.0\% & 0.882 & -- & -- & -- & \underline{0.934} & 0.712 & 0.920 & 0.884 & -- & -- & -- & \textbf{0.959} & 0.833 & -- & -- & 0.933 & 0.855 \\
\texttt{donors} & 619326 & 10 & 5.9\% & 0.768 & -- & -- & -- & 0.824 & 0.742 & 0.888 & 0.815 & -- & -- & -- & 0.733 & \textbf{0.968} & -- & -- & 0.628 & \underline{0.909} \\
\texttt{fraud} & 284807 & 29 & 0.2\% & 0.951 & -- & -- & -- & 0.953 & \underline{0.954} & 0.950 & 0.947 & -- & -- & -- & \textbf{0.956} & 0.621 & -- & -- & 0.947 & 0.941 \\
\texttt{http} & 567498 & 3 & 0.4\% & \textbf{0.999} & -- & -- & -- & 0.996 & 0.994 & 0.979 & 0.992 & -- & -- & -- & \underline{0.998} & 0.281 & -- & -- & 0.992 & 0.994 \\
\texttt{magic.gamma} & 19020 & 10 & 35.2\% & 0.719 & 0.693 & -- & \textbf{0.811} & 0.667 & 0.712 & 0.638 & 0.681 & \underline{0.784} & 0.638 & 0.746 & 0.735 & 0.649 & -- & 0.722 & 0.744 & 0.681 \\
\texttt{mammography} & 11183 & 6 & 2.3\% & 0.861 & 0.719 & -- & 0.838 & 0.894 & 0.830 & \textbf{0.906} & \underline{0.905} & -- & 0.717 & 0.805 & 0.824 & 0.891 & 0.741 & 0.865 & 0.843 & 0.881 \\
\texttt{shuttle} & 49097 & 9 & 7.2\% & \textbf{0.997} & -- & -- & -- & 0.990 & 0.984 & 0.993 & 0.995 & -- & -- & -- & 0.989 & 0.641 & -- & -- & 0.972 & \underline{0.996} \\
\texttt{skin} & 245057 & 3 & 20.8\% & 0.667 & -- & -- & -- & 0.417 & 0.595 & 0.489 & 0.471 & -- & -- & -- & 0.722 & 0.432 & -- & -- & \underline{0.727} & \textbf{0.894} \\
\texttt{smtp} & 95156 & 3 & 0.0\% & 0.904 & -- & -- & -- & 0.808 & 0.800 & 0.880 & 0.912 & -- & -- & -- & \underline{0.925} & 0.829 & -- & -- & 0.918 & \textbf{0.956} \\
\texttt{backdoor} & 95329 & 193 & 2.4\% & 0.742 & -- & -- & -- & -- & 0.740 & 0.846 & 0.789 & -- & -- & -- & \textbf{0.905} & 0.237 & -- & -- & 0.733 & \underline{0.894} \\
\texttt{campaign} & 41188 & 62 & 11.3\% & 0.699 & 0.625 & -- & 0.741 & -- & 0.766 & 0.770 & \textbf{0.783} & 0.724 & -- & -- & 0.731 & 0.500 & -- & -- & \underline{0.781} & 0.667 \\
\texttt{census} & 299285 & 500 & 6.2\% & 0.607 & -- & -- & -- & -- & -- & -- & -- & -- & -- & -- & \underline{0.665} & -- & -- & -- & 0.532 & \textbf{0.672} \\
\texttt{InternetAds} & 1966 & 1555 & 18.7\% & \textbf{0.701} & 0.609 & -- & 0.652 & 0.615 & \underline{0.696} & 0.677 & 0.676 & 0.585 & 0.614 & 0.562 & 0.610 & 0.523 & -- & 0.616 & 0.648 & 0.619 \\
\texttt{mnist} & 7603 & 78 & 9.2\% & 0.809 & 0.673 & -- & 0.837 & \underline{0.850} & 0.576 & 0.746 & 0.774 & 0.755 & 0.635 & 0.675 & \textbf{0.858} & 0.426 & 0.643 & 0.798 & 0.621 & 0.849 \\
\texttt{musk} & 3062 & 166 & 3.2\% & \textbf{1.000} & 0.636 & -- & 0.616 & \textbf{1.000} & \textbf{1.000} & 0.956 & 0.946 & 0.050 & 0.560 & 0.601 & \textbf{1.000} & \underline{0.992} & 0.915 & 0.071 & \textbf{1.000} & \textbf{1.000} \\
\texttt{optdigits} & 5216 & 62 & 2.9\% & 0.733 & 0.537 & -- & 0.395 & 0.515 & \textbf{0.869} & 0.605 & 0.682 & 0.513 & 0.511 & 0.512 & 0.615 & 0.608 & 0.569 & 0.323 & \underline{0.794} & 0.576 \\
\texttt{SpamBase} & 4207 & 57 & 39.9\% & 0.627 & 0.457 & -- & 0.489 & 0.550 & \underline{0.662} & 0.656 & \textbf{0.688} & 0.400 & 0.447 & 0.500 & 0.543 & 0.372 & 0.566 & 0.495 & 0.650 & 0.582 \\
\texttt{speech} & 3686 & 400 & 1.7\% & 0.466 & 0.509 & -- & 0.491 & 0.469 & 0.476 & 0.470 & 0.491 & \textbf{0.763} & 0.532 & \underline{0.570} & 0.478 & 0.496 & -- & 0.451 & 0.472 & 0.446 \\
\midrule
Avg.\ AUC & & & & \underline{0.762} & 0.658 & 0.730 & 0.729 & 0.741 & 0.743 & 0.744 & 0.748 & 0.655 & 0.624 & 0.688 & 0.759 & 0.655 & 0.727 & 0.708 & 0.750 & \textbf{0.777} \\
Avg.\ rank & & & & \underline{6.04} & 10.11 & 8.80 & 7.54 & 7.81 & 7.40 & 7.79 & 7.23 & 10.29 & 11.57 & 10.14 & 6.71 & 10.13 & 8.65 & 8.69 & 7.43 & \textbf{5.64} \\
Share of wins & & & & 5.4 & 0 & 0 & 2 & 2.2 & 1.2 & 2.2 & 6.2 & 1 & 2 & 1 & \underline{7.2} & 2 & 0 & 6 & 1.2 & \textbf{7.4} \\
\bottomrule
\end{tabular}

\end{table*}

\subsection{Pixel-level explanation on raw images}
WAND's score is differentiable in its input features, so when those
features come from a differentiable encoder $\phi$, $\partial(\text{WAND
score})/\partial(\text{pixels})$ is obtained by autograd through $\phi$,
with no training. We embed the \textsc{AnoCUB} images with a frozen
pre-trained ResNet-18, fit WAND on the $512$-d embeddings (z-scored, no whitening; detection AUC
$0.99$), and back-propagate the anomaly score of a flagged bird to its
pixels (SmoothGrad). Fig.~\ref{fig:pixsal} in the main text shows the
resulting saliency for three flagged anomalies and a normal inlier on a
shared scale: it concentrates on the anomalies' discriminative
body/bill regions and is weak on the inlier. This realises the
pixel-level explanation; embedding-space detection (AUC $0.99$) is on
par with the named concept space ($1.0$), the trade-off being only that
embedding dimensions are not individually named, so the map is
spatial-only. Note that, as in the concept space, we fit WAND
\emph{without} Mahalanobis whitening: on the contaminated $512$-d
embedding covariance, whitening suppresses the anomaly directions and
drops AUC to $0.88$.

\subsection{The \textsc{AnoCUB} dataset}
\textsc{AnoCUB} is the explainable-AD benchmark we derive from
CUB-200-2011~\cite{cub2011}. Each bird image is represented by its $312$
named binary attributes (the class-level concept profile); inliers are
the sparrow species and anomalies are a small set of birds from very
different families (pelican, frigatebird, mallard, hummingbird), giving
a $1259\times312$ task with $15$ anomalies. As \textsc{AnoCUB} is not a
standard download, we release the construction script, which regenerates
the exact task (\texttt{anocub\_task.npz}) from the public CUB archive,
together with the embedding/saliency scripts below.

\paragraph{Whitening ablation.} On the ResNet-18 embedding used for the
pixel-level view (Fig.~\ref{fig:pixsal}), Mahalanobis whitening on the
contaminated covariance suppresses the anomaly directions; turning it
off recovers detection (Table~\ref{tab:anocub-whiten}), no backbone
change is needed.

\begin{table}[h]
\centering
\caption{\textsc{AnoCUB} detection AUC: ResNet-18 embedding, with/without
Mahalanobis whitening.}
\label{tab:anocub-whiten}
\begin{tabular}{lcc}
\toprule
ResNet-18 embedding & whiten=ON & whiten=OFF \\
\midrule
raw       & 0.882 & 0.980 \\
z-scored  & 0.881 & \textbf{0.989} \\
\bottomrule
\end{tabular}
\end{table}

\paragraph{Backbone.} Among lightweight ImageNet backbones (z-scored
embeddings, no whitening), ResNet-18 gives the best embedding-space
detection (Table~\ref{tab:anocub-bb}); we therefore keep it for the
pixel-level explanation.

\begin{table}[h]
\centering
\caption{\textsc{AnoCUB} detection AUC by lightweight backbone (z-scored,
no whitening).}
\label{tab:anocub-bb}
\begin{tabular}{lcc}
\toprule
backbone & dim & AUC \\
\midrule
\textbf{ResNet-18}     & 512  & \textbf{0.989} \\
EfficientNet-B0        & 1280 & 0.975 \\
MobileNetV3-small      & 576  & 0.971 \\
\bottomrule
\end{tabular}
\end{table}

\subsection{Explanation gallery and failure modes}
Figure~\ref{fig:anocub-gallery} illustrates both explanation modes on the
same task and is deliberately fair about where each is weak.

\begin{figure*}[t]
\centering
\includegraphics[width=0.86\textwidth]{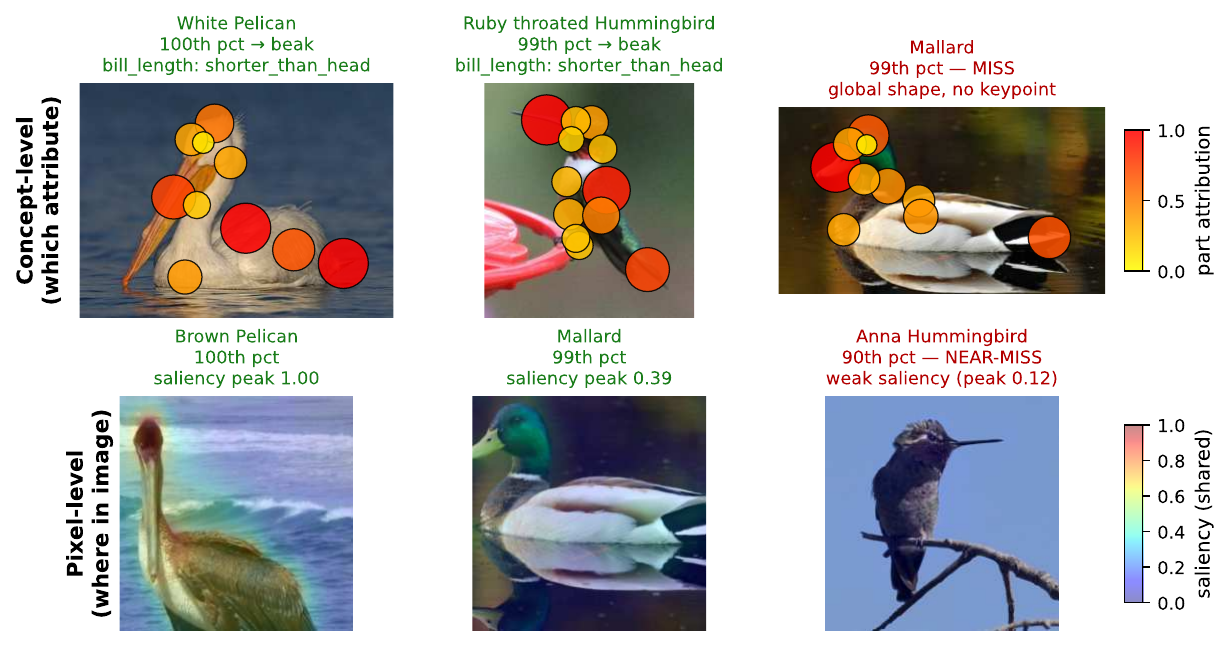}
\caption{\textsc{AnoCUB} explanation gallery and failure modes.
\emph{Top:} concept-level part heatmaps (witness attribution mapped to
CUB keypoints, hotter = more responsible). \emph{Bottom:} pixel-level
saliency through a frozen ResNet-18 (shared scale). Green titles = clean,
red = where that mode is weak. The two modes have \emph{complementary}
blind spots. The small Anna Hummingbird the embedding nearly misses
(bottom right, $90$th pct: a perched hummingbird fills a sparrow-sized
region and lands near the inlier manifold, so its saliency is weak and
diffuse) is flagged cleanly in concept space (top middle, $99$th pct,
localised to the bill). Conversely the Mallard whose single
most-responsible concept is the \emph{global} ``duck-like'' shape, which
has no body-part keypoint, so the part heatmap cannot localise it (top
right), is sharply localised by pixel saliency (bottom middle).}
\label{fig:anocub-gallery}
\end{figure*}

\paragraph{What is, and is not, missed.} In the named-concept space all
$15$ anomalies rank in the top $15$ (detection AUC $1.0$), so there is no
detection miss; $14/15$ attribute most strongly to bill length, which
maps to the beak keypoint and localises cleanly. The lone exception is a
Mallard whose top concept is the global ``duck-like'' shape ($\approx
15\%$ of its attribution mass falls on global attributes with no
keypoint): the part-grounded view cannot pin it to a location, even
though the named reason is still correct and human-readable. In the
ResNet-18 embedding space detection is near-perfect (AUC $0.99$): the
cross-family large birds (pelican, frigatebird, mallard) are flagged at
the $99$th percentile with crisp saliency, but the small hummingbirds are
harder, one Anna Hummingbird drops to the $90$th percentile (about
$130$ sparrow inliers score above it) because at $224{\times}224$ it
embeds close to the sparrows. We report these openly: neither mode
dominates, and the concept- and pixel-level views are complementary
rather than redundant, the case each one struggles with is handled
by the other.

\subsection{Case study: named-feature medical data}
On Breast Cancer Wisconsin (benign cases as inliers, a few malignant
cases as anomalies; AUC $0.933$, no labels used), WAND attributes
flagged malignant cases to \texttt{mean radius}, \texttt{mean/worst
perimeter}, \texttt{worst area}, and \texttt{mean concavity}, the
clinically recognised malignancy markers, read directly off the
witness directions (Figure~\ref{fig:e4}).
\begin{figure}[h]
\centering
\includegraphics[width=\linewidth]{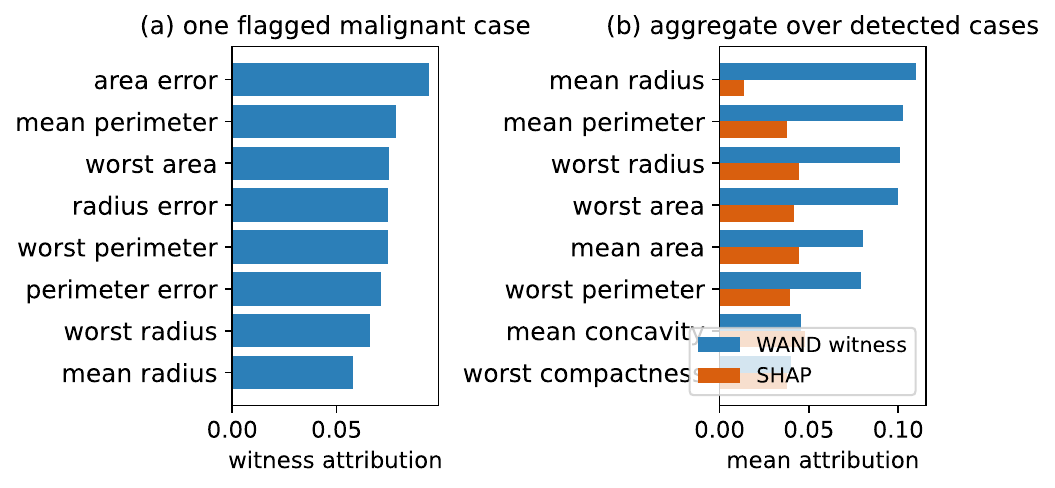}
\caption{Breast Cancer Wisconsin. (a) Witness attribution for one
flagged malignant case; (b) aggregate over detected cases, WAND witness
vs.\ SHAP, both concentrate on size/shape malignancy markers.}
\label{fig:e4}
\end{figure}

\subsection{Critical-difference diagram (AUPR / AP)}
The main text ranks the detectors by ROC-AUC; to check that the
ordering is not an artefact of that metric, Figure~\ref{fig:cd-ap}
repeats the Nemenyi post-hoc analysis on AUPR / average precision, which
weighs the positive (anomaly) class more heavily. The two rankings agree.
\begin{figure}[h]
\centering
\includegraphics[width=\linewidth]{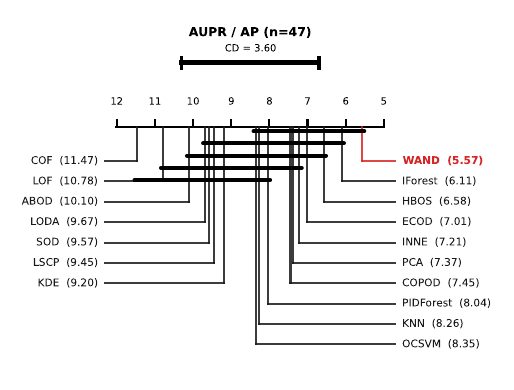}
\caption{Critical-difference diagram on AUPR / AP (Nemenyi post-hoc,
$\alpha=0.05$, $\mathrm{CD}=3.60$); the ROC-AUC diagram is
Figure~\ref{fig:cd} in the main text. The AP ranking mirrors the ROC
ordering.}
\label{fig:cd-ap}
\end{figure}

\subsection{AUC vs.\ runtime}
Beyond accuracy alone, Figure~\ref{fig:time-auc} plots detection quality
against wall-clock cost, restricted to the $22$ datasets that
\emph{every} method completes so the per-method means cover the same
tasks. \WAND{} lands on the top-left Pareto frontier: no baseline is at
once faster and more accurate on this shared subset.
\begin{figure}[h]
\centering
\includegraphics[width=0.85\linewidth]{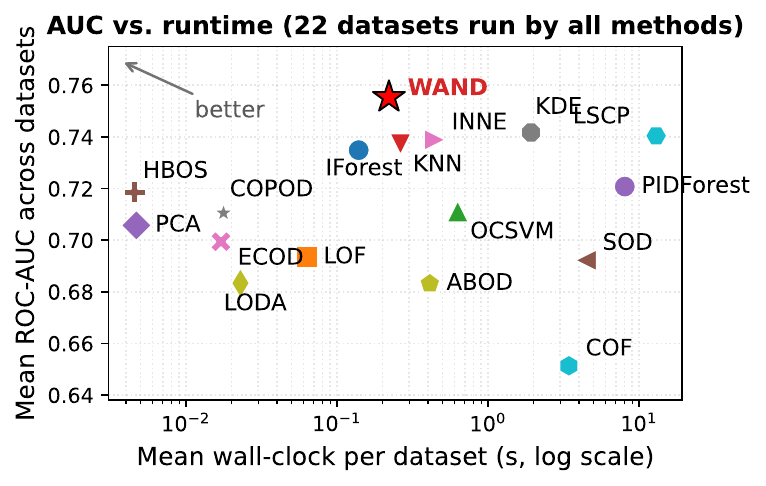}
\caption{Mean ROC-AUC vs.\ mean wall-clock per dataset (log $x$),
averaged over the $22$ ADBench datasets completed by \emph{every} method
so all means cover the same tasks. \WAND{} sits on the AUC--runtime
Pareto frontier (top-left): no other method is both faster and more
accurate on this shared subset. The subset necessarily
excludes the largest datasets, on which the neighbour-based methods (LOF,
kNN, ABOD, COF, SOD, LSCP, KDE, OCSVM) exceed the time budget, whereas
\WAND{} runs on all $47$ (Table~\ref{tab:main}).}
\label{fig:time-auc}
\end{figure}

\subsection{Scalability to large $n$}
\WAND{} is \emph{inductive}: it calibrates once on a reference sample
(per-direction median/MAD, weights, and the sub-Gaussian baseline), then
scores arbitrary points in a single pass. The probe budget $K$ is
\emph{independent of $n$}, so the deployed scorer is $O(Knd)$, linear
in $n$, the same complexity class as Isolation Forest / ECOD / HBOS and
unlike the $O(n^{2})$ neighbour methods. Points are scored independently
against the frozen background, so scoring streams in constant-memory
chunks (bit-for-bit identical to a single call: $\max|\Delta|=0$ on
$2{\times}10^{5}$ points) and is embarrassingly parallel over rows.
Figure~\ref{fig:scaling} confirms this on synthetic data ($d{=}20$, an
$8$-core CPU): after a $1.1$\,s calibration on $50$k points, the default
$K{=}1024$ scorer processes $10^{7}$ points in $164$\,s
($\approx 6.1{\times}10^{4}$ pts/s) and the lite $K{=}256$ setting in
$47$\,s ($\approx 2.1{\times}10^{5}$ pts/s); throughput is flat in $n$, so
wall-clock grows strictly linearly out to $n{=}10^{7}$. The probe budget
$K$ is therefore a direct speed/quality knob, and \WAND{} scales to
high-throughput streams, the probe-efficiency guarantee is a budget,
not a runtime, claim, but the runtime itself is linear and small.

\begin{figure*}[t]
\centering
\includegraphics[width=0.92\textwidth]{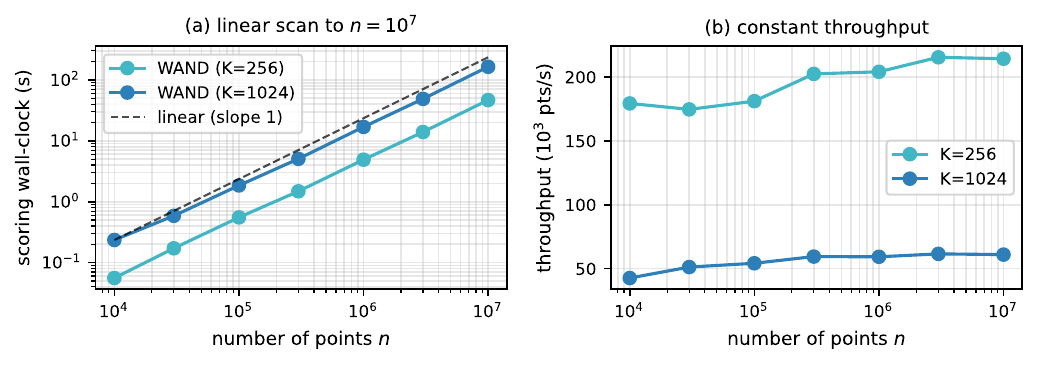}
\caption{Large-$n$ scaling of the \WAND{} scorer (synthetic, $d{=}20$,
$8$-core CPU). (a) Scoring wall-clock vs.\ $n$ tracks the slope-1 (linear)
reference for both probe budgets; (b) throughput is constant in $n$, so
\WAND{} scores $10^{7}$ points in $47$\,s ($K{=}256$) / $164$\,s
($K{=}1024$) after a one-time $\sim\!1$\,s calibration.}
\label{fig:scaling}
\end{figure*}

\subsection{Dataset statistics}
Table~\ref{tab:datasets} lists the 47 ADBench tasks.
\begin{table*}[t]
  \centering
  \footnotesize
  \setlength{\tabcolsep}{5pt}
  \renewcommand{\arraystretch}{1.00}
  \caption{The 47 ADBench tabular datasets used in our experiments,
  grouped by scale. \emph{Samples} ($n$), \emph{Features} ($d$),
  \emph{\#Anom.}\ (number of labelled outliers) and \emph{\%Anom.}
  (their fraction) are taken from the public ADBench
  Classical mirror. Datasets are referred to by short codes in the
  rest of the paper; the numeric prefix matches the ADBench naming
  convention.}
  \label{tab:datasets}
  \begin{tabular}{@{}lrrrrr@{\hskip 18pt}lrrrrr@{}}
    \toprule
    Dataset & $n$ & $d$ & \#Anom.\ & \%Anom.\ & Code  &
    Dataset & $n$ & $d$ & \#Anom.\ & \%Anom.\ & Code \\
    \midrule
    \multicolumn{12}{l}{\textit{Small-scale ($n \leq 1{,}000$, 12 datasets)}} \\
    \cmidrule(lr){1-12}
    breastw       &   683 &    9 &  239 & 34.99\% &  4 & glass         &   214 &    7 &     9 &  4.21\% & 14 \\
    Hepatitis     &    80 &   19 &   13 & 16.25\% & 15 & Ionosphere    &   351 &   32 &   126 & 35.90\% & 18 \\
    Lymphography  &   148 &   18 &    6 &  4.05\% & 21 & Pima          &   768 &    8 &   268 & 34.90\% & 29 \\
    Stamps        &   340 &    9 &   31 &  9.12\% & 37 & vertebral     &   240 &    6 &    30 & 12.50\% & 39 \\
    WBC           &   223 &    9 &   10 &  4.48\% & 42 & WDBC          &   367 &   30 &    10 &  2.72\% & 43 \\
    wine          &   129 &   13 &   10 &  7.75\% & 45 & WPBC          &   198 &   33 &    47 & 23.74\% & 46 \\
    \midrule
    \multicolumn{12}{l}{\textit{Medium-scale ($1{,}000 < n \leq 10{,}000$, 15 datasets)}} \\
    \cmidrule(lr){1-12}
    annthyroid    &  7200 &    6 &  534 &  7.42\% &  2 & cardio        &  1831 &   21 &   176 &  9.61\% &  6 \\
    Cardiotocogr. &  2114 &   21 &  466 & 22.04\% &  7 & fault         &  1941 &   27 &   673 & 34.67\% & 12 \\
    landsat       &  6435 &   36 & 1333 & 20.71\% & 19 & letter        &  1600 &   32 &   100 &  6.25\% & 20 \\
    PageBlocks    &  5393 &   10 &  510 &  9.46\% & 27 & pendigits     &  6870 &   16 &   156 &  2.27\% & 28 \\
    satellite     &  6435 &   36 & 2036 & 31.64\% & 30 & satimage-2    &  5803 &   36 &    71 &  1.22\% & 31 \\
    thyroid       &  3772 &    6 &   93 &  2.47\% & 38 & vowels        &  1456 &   12 &    50 &  3.43\% & 40 \\
    Waveform      &  3443 &   21 &  100 &  2.90\% & 41 & Wilt          &  4819 &    5 &   257 &  5.33\% & 44 \\
    yeast         &  1484 &    8 &  507 & 34.16\% & 47 &               &       &      &       &        &    \\
    \midrule
    \multicolumn{12}{l}{\textit{Large-scale ($n > 10{,}000$, 11 datasets)}} \\
    \cmidrule(lr){1-12}
    ALOI          & 49534 &   27 & 1508 &  3.04\% &  1 & celeba        & 202599 &   39 &  4547 &  2.24\% &  8 \\
    cover         &286048 &   10 & 2747 &  0.96\% & 10 & donors        & 619326 &   10 & 36710 &  5.93\% & 11 \\
    fraud         &284807 &   29 &  492 &  0.17\% & 13 & http          & 567498 &    3 &  2211 &  0.39\% & 16 \\
    magic.gamma   & 19020 &   10 & 6688 & 35.16\% & 22 & mammography   &  11183 &    6 &   260 &  2.32\% & 23 \\
    shuttle       & 49097 &    9 & 3511 &  7.15\% & 32 & skin          & 245057 &    3 & 50859 & 20.75\% & 33 \\
    smtp          & 95156 &    3 &   30 &  0.03\% & 34 &               &        &      &       &        &    \\
    \midrule
    \multicolumn{12}{l}{\textit{High-dimensional ($d \gg 1$, 9 datasets)}} \\
    \cmidrule(lr){1-12}
    backdoor      & 95329 &  196 & 2329 &  2.44\% &  3 & campaign      &  41188 &   62 &  4640 & 11.27\% &  5 \\
    census        &299285 &  500 &18568 &  6.20\% &  9 & InternetAds   &   1966 & 1555 &   368 & 18.72\% & 17 \\
    mnist         &  7603 &  100 &  700 &  9.21\% & 24 & musk          &   3062 &  166 &    97 &  3.17\% & 25 \\
    optdigits     &  5216 &   64 &  150 &  2.88\% & 26 & SpamBase      &   4207 &   57 &  1679 & 39.91\% & 35 \\
    speech        &  3686 &  400 &   61 &  1.65\% & 36 &               &        &      &       &        &    \\
    \bottomrule
  \end{tabular}
\end{table*}

\section{Additional Clarifications}
This section collects detailed answers to natural questions about cost,
calibration, and scope. We are explicit about what we have measured and
what we leave to future work.

\paragraph{Cost of the spacings step and memory footprint.}
For each of the $K$ directions we (i) form the projections
$z_{i}=u^{\top}x_{i}$ in $O(nd)$ and (ii) sort them \emph{once},
$O(n\log n)$. The single sort yields all order statistics
$z_{(1)}\le\cdots\le z_{(n)}$, and from them every two-sided $k$-spacing
$d_{k,i}$ in one $O(n)$ pass over ranks; the MAD branch reuses the same
sorted array for the median/MAD, so no second sort is needed. The
per-direction cost is therefore $O\bigl(n(d+\log n)\bigr)$ and the total
$O\bigl(Kn(d+\log n)\bigr)$ (Remark~\ref{rem:runtime}); the $d$
axis probes reuse the coordinates directly. The working set is $O(n)$ per
direction (the projected/sorted vector), and directions are processed
independently, so the footprint is $O(n)$ beyond the $O(nd)$ input rather
than $O(Kn)$, and scoring streams in constant-memory chunks.

\paragraph{The null gate $q_{0}$, contamination, and heavy tails.}
$q_{0}$ enters only as a per-direction weight gate
$(\Delta(u_{k})-q_{0})_{+}$ in \eqref{eq:score}: it reweights
\emph{directions} and never enters the per-point statistic $\tau_{i}$,
whose breakdown is carried by the median/MAD (Theorem~\ref{thm:robust}),
not by $q_{0}$. We keep $q_{0}$ a Ledoit--Wolf Gaussian-copy quantile
because it serves only as a relative reference, and the aggregation falls
back to a uniform-weighted mean when the gate is empty. We did not run
robust-covariance or permutation-based nulls; since $q_{0}$ only orders
directions while robustness is decoupled into the MAD statistic, we
expect low sensitivity, consistent with \WAND{} remaining effective on
heavy-tailed inliers (main text), but a systematic robust-null study is
future work.

\paragraph{High-dimensional regime and witness-cone sizes.}
The worst-case probe bound $p_{\tau}=\Theta((\tau/\sqrt d)^{d-1})$
(Lemma~\ref{lem:cap}) is loose in high $d$, yet a fixed uniform budget
$K=1024$ suffices empirically up to the highest-dimensional ADBench
tasks ($d$ up to ${\approx}1.5$k), direct evidence that real witness
cones are far wider than the adversarial cap. We tested adaptive
(posterior-targeting) direction sampling and saw no measurable mean-AUC
gain over uniform draws (Section~\ref{sec:posterior}), so we keep uniform
sampling. A structure-adaptive scheme that aligns probes with the
empirical principal directions is a promising way to tighten high-$d$
behaviour and is among our stated future directions; we have not yet
evaluated it.

\paragraph{Why the spacings penalty is small, and adaptive switching.}
The MAD and spacings components are combined per point and per direction
by an evidence-disjunction
$\tau_{i}=\max(\tau^{\mathrm{mad}},s_{u}\tau^{\mathrm{spc}})$
\eqref{eq:tau-comb}, so the spacings term contributes only where its
rescaled signal exceeds the MAD term. This maximum already acts as a
local, per-direction switch, which is why enabling spacings shifts mean
AUC only slightly. We did not add a dataset-level gate; a cheap
multimodality test (e.g.\ a dip statistic on $u^{\top}X$) that enables
spacings only on multi-modal projections is a sensible heuristic to
remove the residual penalty, which we leave to future work.

\paragraph{Affine transformations and the axis pathway.}
Features are standardised (per-feature centring and scaling) before
probing, so per-feature rescaling does not change the scores, and the
uniform-sphere pathway is rotation-equivariant up to the per-direction
MAD calibration. The axis pathway is, by construction, \emph{not}
affine-equivariant: it probes the canonical axes $e_{1},\dots,e_{d}$
precisely to catch single-feature anomalies that rotation-invariant
probes miss (the marginal regime), so we trade equivariance for that
coverage and keep it secondary through $\lambda=1/4$ \eqref{eq:mix}.
Whitening before probing is not uniformly helpful: on contaminated data
it can suppress the very anomaly directions, e.g.\ on the \textsc{AnoCUB}
embedding Mahalanobis whitening drops AUC from $0.98$ to $0.88$
(Table~\ref{tab:anocub-whiten}); we therefore do not whiten by default.

\paragraph{Sensitivity of the faithfulness protocol.}
We follow the standard deletion/insertion protocol with per-feature
median replacement~\cite{petsiuk2018rise}, and report the bounded,
self-normalised score insertion${-}$deletion${\in}[-1,1]$ averaged over
$33$ datasets, which limits the influence of any single masking choice.
We did not sweep alternative imputations (mean, marginal resampling) or
feature-subset sizes; such a sensitivity analysis would further
strengthen the faithfulness conclusions and is straightforward future
work.

\paragraph{The differentiable score as a training loss.}
We use the differentiable \WAND{} score only to produce gradient
explanations, including pixel saliency through a \emph{frozen} encoder
(Fig.~\ref{fig:pixsal}); we did not back-propagate it into the encoder
as a training objective. Using the score to shape upstream
representations (a learned witness encoder) is an explicit future
direction, and we report no empirical results on it here.

\paragraph{Empirical explanation coverage.}
The coverage guarantee (every $\tau$-margin anomaly has at least one
witness) is realised in the attribution by the dominant-witness set. On
\textsc{AnoCUB} all $15$ anomalies are covered (detection AUC $1.0$) and
$14/15$ concentrate on a single dominant witness (bill length), the lone
exception spreading over global-shape attributes
(Fig.~\ref{fig:anocub-gallery}). We do not yet report suite-wide
distributions of the active-witness count or of the contribution mass
$\gamma_{i,k}$; tabulating, per dataset, the fraction of flagged points
with at least one above-threshold witness would quantify coverage
directly and is a useful addition we flag for future work.

\paragraph{Parameter robustness ($\lambda$, $K$, $k$).}
The defaults are insensitive over wide ranges. Mean AUC is flat
($\pm0.003$) for $\lambda\in[0.1,0.5]$, and we fix $\lambda=1/4$. The
budget $K$ is a monotone speed/quality knob whose AUC saturates around
$K{\approx}4|H|$, with $|H|$ the soft-extreme hull size, so $K=1024$
sits on the plateau for the suite. The spacings order
$k=\lceil\sqrt n\rceil$ is the classical parameter-free
choice~\cite{loftsgaarden1965density}. Table~\ref{tab:knob-sensitivity}
above tabulates the underlying per-knob sweep.

\paragraph{Comparison to recent density detectors.}
Our sixteen unsupervised baselines already include strong
density and marginal detectors (ECOD, COPOD, HBOS, kernel density,
LODA); we did not include very recent variants such as VSDE, so we make
no claim against them. Because the witness attribution needs only a
differentiable score, it could in principle wrap any differentiable
density model to add explanations, an interesting direction we have not
evaluated.

\section{Additional Experiments (Post-Acceptance)}
This section reports additional experiments; these results extend,
rather than revise, the main text above.

\subsection{A native explanation baseline for a second detector}
Our explanation-quality comparison (Table~\ref{tab:xai}) applies
SHAP/LIME only to \WAND{} itself; a natural complement is a
\emph{native} explainer of a different strong detector. We add
Isolation Forest's shortest-isolation-path structure as exactly that
baseline.

\paragraph{Method.} For a query point flagged by a fitted Isolation Forest,
we walk its decision path in every tree; each internal node the path
crosses splits on one feature, and we weight that feature's vote by
$1/(\mathrm{depth}+1)$ so splits near the root, the ones responsible for
the short isolation path, count more than deep splits. Votes are averaged
over trees and $\ell_1$-normalised per point. Like ECOD's native
attribution, this is read off the fitted ensemble directly, at zero extra
model queries.

\paragraph{Results.} Table~\ref{tab:xai} includes this new row
(\emph{IForest (native)}), evaluated under the identical protocols:
synthetic ground-truth attribution-AUC (Section~\ref{sec:exp-xai}) and
real-data deletion/insertion faithfulness on the same $33$ ADBench
datasets used for the other rows. Isolation Forest's native attribution
is close to chance on the synthetic ground truth ($0.498$/$0.499$
attr-AUC, axis/oblique, vs.\ $0.5$ for a random ranking) and only
marginally above random on real-data faithfulness ($0.031$ mean, beating
SHAP on $3\%$ of datasets, the same rate as the random baseline). This is
a real, non-cherry-picked negative result for this particular baseline,
not evidence that no native explainer can compete with \WAND{}'s:
depth-weighted split voting is the simplest reading of an Isolation
Forest, and more elaborate model-specific importances for tree ensembles
exist (e.g.\ DIFFI~\cite{apx:diffi}) that we have not evaluated. What the
result does show is that \emph{simply being native to a strong detector
is not sufficient} for a faithful explanation, ECOD's native attribution
(also in Table~\ref{tab:xai}) is far stronger than Isolation Forest's, so
the comparison is detector-specific rather than a generic
native-vs-post-hoc effect, and \WAND{}'s witness attribution remains the
strongest native explainer we have tested by a wide margin on both
protocols.

\subsection{Per-knob sensitivity table}
Table~\ref{tab:knob-sensitivity} reports mean ROC-AUC over all $47$
ADBench datasets for each knob swept individually around its default
(default in \textbf{bold}), extending the qualitative ``Parameter
robustness'' discussion above with the underlying numbers and justifying
the specific default configuration ($K{=}1024$, spacing
$k{=}\lceil\sqrt n\rceil$, axis probes on, $\lambda{=}1/4$). All four
knobs are flat to within $0.015$ mean AUC across their tested ranges,
confirming quantitatively that the single fixed default used throughout
the paper is not a narrow optimum: no per-dataset tuning is needed to
reach the reported accuracy.
\begin{table}[h]
\centering
\caption{Mean ROC-AUC (47 ADBench datasets) as each knob is swept with all
others held at their default. Default setting in \textbf{bold}.}
\label{tab:knob-sensitivity}
\begin{tabular}{lcc}
\toprule
Knob & Setting & Mean ROC-AUC (47 datasets) \\
\midrule
\multirow{4}{*}{Probe budget $K$}
  & 256 & 0.7686 \\
  & 512 & 0.7640 \\
  & \textbf{1024} & \textbf{0.7726} \\
  & 2048 & 0.7779 \\
\midrule
\multirow{5}{*}{Axis mix weight $\lambda$}
  & 0.0 & 0.7676 \\
  & 0.1 & 0.7701 \\
  & \textbf{0.25} & \textbf{0.7726} \\
  & 0.5 & 0.7743 \\
  & 1.0 & 0.7755 \\
\midrule
\multirow{4}{*}{Spacing order factor}
  & 0.0 & 0.7716 \\
  & 0.5 & 0.7732 \\
  & \textbf{1.0} & \textbf{0.7726} \\
  & 1.5 & 0.7734 \\
\midrule
\multirow{4}{*}{Langevin refinement steps}
  & \textbf{0} & \textbf{0.7768} \\
  & 10 & 0.7712 \\
  & 20 & 0.7726 \\
  & 40 & 0.7725 \\
\bottomrule
\end{tabular}

\end{table}

\subsection{Compact wall-clock runtime table}
Figure~\ref{fig:time-auc} above plots accuracy against wall-clock cost;
Table~\ref{tab:runtime-compact} gives the same comparison in tabular form
for easier reference, mean ROC-AUC and mean wall-clock over the
$22$ ADBench datasets every method in the main benchmark completes (same
subset as Figure~\ref{fig:time-auc}), sorted by runtime. \WAND{} attains
the best mean AUC of all $17$ methods on this shared subset in
$0.220$\,s, and the fastest baseline within $0.02$ AUC of it, IForest, is
only $1.6\times$ faster; every other method is both slower and less
accurate.
\begin{table}[h]
\centering
\caption{Mean ROC-AUC vs.\ mean wall-clock, over the $22$ ADBench datasets
completed by every method (same subset as Figure~\ref{fig:time-auc}),
sorted by runtime.}
\label{tab:runtime-compact}
\begin{tabular}{lcc}
\toprule
Method & Mean ROC-AUC & Mean wall-clock (s) \\
\midrule
HBOS & 0.719 & 0.005 \\
PCA & 0.706 & 0.005 \\
ECOD & 0.699 & 0.017 \\
COPOD & 0.710 & 0.018 \\
LODA & 0.683 & 0.023 \\
LOF & 0.693 & 0.063 \\
IForest & 0.735 & 0.139 \\
\textbf{WAND} & 0.755 & 0.220 \\
KNN & 0.737 & 0.263 \\
ABOD & 0.683 & 0.411 \\
INNE & 0.739 & 0.445 \\
OCSVM & 0.711 & 0.629 \\
KDE & 0.742 & 1.925 \\
COF & 0.651 & 3.429 \\
SOD & 0.692 & 4.444 \\
PIDForest & 0.721 & 8.036 \\
LSCP & 0.740 & 12.968 \\
\bottomrule
\end{tabular}

\end{table}

{\renewcommand{\refname}{Appendix References}
}

\end{document}